\documentclass{article}

\usepackage{microtype}
\usepackage{graphicx}
\usepackage{subcaption}
\usepackage{booktabs} 
\usepackage{multirow}
\usepackage{booktabs}
\usepackage{comment}
\usepackage{xcolor}
\usepackage{multirow} 
\usepackage{newunicodechar}
\DeclareUnicodeCharacter{2060}{}
\usepackage{hyperref}

\usepackage[accepted]{icml2026}

\usepackage{amsmath}
\usepackage{amssymb}
\usepackage{mathtools}
\usepackage{amsthm}
\usepackage{comment}
\usepackage{pgfplots}
\usepackage{tikz}
\usetikzlibrary{arrows.meta,positioning,calc}
\usepackage{tikz}

\usepackage[capitalize,noabbrev]{cleveref}

\theoremstyle{plain}
\newtheorem{theorem}{Theorem}[section]
\newtheorem{proposition}[theorem]{Proposition}
\newtheorem{lemma}[theorem]{Lemma}

\theoremstyle{definition}
\newtheorem{definition}[theorem]{Definition}

\theoremstyle{remark}
\newtheorem{remark}[theorem]{Remark}

\usepackage[textsize=tiny]{todonotes}

\icmltitlerunning{Set-Preserving Calibration from Conformal P-Values to E-Values}

\begin{document}

\twocolumn[
  \icmltitle{⁠⁠Set-Preserving Calibration from Conformal P-Values to E-Values}



  \icmlsetsymbol{equal}{*}

  \begin{icmlauthorlist}
    \icmlauthor{Nabil Alami}{mbz}
    \icmlauthor{Jad Zakharia}{ens}
    \icmlauthor{Souhaib Ben Taieb}{mbz,mons}
    
  \end{icmlauthorlist}

  \icmlaffiliation{mbz}{Department of Statistics and Data Science, MBZUAI}
  \icmlaffiliation{ens}{École Normale Superieure de Cachan}
  \icmlaffiliation{mons}{Department of Computer Science, University of Mons}

  \icmlcorrespondingauthor{Nabil Alami}{nabil.alami@mbzuai.ac.ae}

  \icmlkeywords{Machine Learning, ICML}

  \vskip 0.3in
  ]



\printAffiliationsAndNotice{}  

\begin{abstract}


Standard conformal prediction (CP) procedures are typically formulated in terms of p-values, but reliance on p-values alone limits flexibility, for example, when combining dependent evidence across models or data splits. Recent work has explored e-value formulations for conformal inference, yet a direct connection between p- and e-value formulations in CP has been missing, especially regarding their statistical efficiency. We first identify limitations of classical p-to-e calibrators in the CP setting, showing that they are not set-preserving and can lead to overly conservative prediction sets. To address this, we propose a novel P2E calibrator that converts conformal p-values into e-values without altering the prediction set induced by the original conformal p-value.  We establish both theoretically and empirically that our calibrator can yield significant efficiency gains over existing p-to-e calibrators. This e-value formulation enables principled use of recent advances in e-value merging and randomization, where we demonstrate its impact in two applications: cross-conformal prediction (CCP), whose variants typically provide only approximate $1-2\alpha$ coverage, and conformal aggregation (CA). In both cases, our e-value-based methods satisfy the desired $1-\alpha$ coverage guarantee while improving efficiency over standard baselines. More broadly, our approach expands the flexibility of CP and opens new directions for efficient, distribution-free uncertainty quantification.


\end{abstract}






Modern machine learning models are increasingly deployed in high-stakes domains where uncertainty is as critical as accuracy. Conformal prediction (CP) has emerged as a versatile framework for constructing prediction sets in both regression and classification tasks \citep{shafer2008tutorial,angelopoulos2023conformal,papadopoulos2002inductive,vovk2005algorithmic}. Given any black-box predictor, conformal methods output a prediction set for a new test input that guarantees finite-sample marginal coverage: 
with user-specified probability (e.g., 90\%), the set contains the true outcome.  

For decades, CP has been grounded in p-values \citep{vovk2025econformal}. Classical p-value–based approaches are often efficient in practice, yielding tight prediction sets while maintaining finite-sample coverage. However, their theoretical flexibility is limited, particularly when aggregating dependent p-values \citep{vovk2022admissible}. In contrast, recent work on e-value–based CP \citep{balinsky2024,gauthier2025values,koning2025optimalconformalpredictionevalues} demonstrates how e-values extend conformal methods to settings less suited for p-values. Originally introduced to the statistical community by \citet{vovk2021values}, \textit{e-values} offer more flexible aggregation rules and additional advantages for hypothesis testing \citep{vovk2021values,vovk2024merging,wang2025only}. Nevertheless, it remains unclear whether e-value–based conformal methods suffer efficiency losses compared to their p-value counterparts, and no unifying principle seems to have been established to connect p-values and e-values within CP.  

To address this gap, we propose a novel framework that links conformal p-values and e-values through a \emph{P2E calibrator}, after exposing some of the limitations of classic p-to-e calibrators \citep{vovk2021values,evaluespoly}. We show how our P2E calibrator is \textit{set-preserving}, which ensures that the e-value formulation of CP is decision-equivalent to the standard p-value formulation: for any test point and miscoverage level, the resulting prediction set is identical. In doing so, our method can yield tight prediction sets while opening the door to the broader e-value theory, including useful extensions based on Markov-type inequalities and randomization results \citep{ramdas2023randomizedexchangeableimprovementsmarkovs}.  

Our framework thus provides a principled way to translate CP from a p-value formulation into an e-value formulation, combining the strengths of both approaches. In our view, p-values and e-values are not competitors but complementary tools that together enrich conformal prediction and its applications.  

In summary, our contributions are as follows:
\begin{enumerate}
\item We study the integration of classic p-to-e calibrators in CP while highlighting their limitations.
    \item We introduce a novel P2E calibrator which transforms the conformal p-variable into an e-variable which preserves the same CP set, and offers additional advantages over classic p-to-e calibrators. 

    \item We leverage e-value theory and Markov-type inequalities combined with our P2E calibrator framework to enhance two settings: cross-conformal prediction (CCP) and conformal aggregation (CA).

    \item We show experimentally that our e-value–based CP methods improve efficiency and flexibility compared to their p-value alternatives.  
\end{enumerate}

\section{Background and notations}

We consider a regression problem where the goal is to predict a real-valued response 
\(y \in \mathcal{Y} = \mathbb{R}\) from an input vector \(x \in \mathcal{X} \subseteq \mathbb{R}^p\). 
We are given a dataset 
\(\mathcal{D} = \{(X_i, Y_i)\}_{i \in \mathcal{I}}\), 
where the pairs \((X_i, Y_i)\) are exchangeable samples from an unknown joint distribution 
\(Q_{X,Y}\) defined over \(\mathcal{X} \times \mathcal{Y}\). We use $|A|$ to denote the Lebesgue measure of $A \subseteq \mathcal{Y}$, or its cardinality if $A$ is discrete. 

Let \(\hat{\mu}:\mathcal{X}\to\mathcal{Y}\) denote a base regressor trained on a subset of \(\mathcal{D}\). 
When the model is trained on  \(\mathcal{D}_1 \subset \mathcal{D}\), we write \(\hat{\mu}(\mathcal{D}_1)\) 
to explicitly indicate the dependence on the training data. Given a miscoverage $\alpha\in(0,1)$, CP allows to construct for any new input \(X_{n+1}\) of unknown response \(Y_{n+1}\) 
a prediction set \(\mathcal{C}(X_{n+1}) \subseteq \mathcal{Y}\) that satisfies the finite-sample coverage guarantee \citep{vovk2005algorithmic}. 

\subsection{Conformal Prediction with p-values}

Split-conformal prediction (SCP)~\citep{papadopoulos2002inductive} is an efficient variant of CP that partitions the dataset \(\mathcal{D}\) into a training set \(\mathcal{D}_{\text{train}}\) and a calibration set \(\mathcal{D}_{\text{cal}}\). Let \(\mathcal{I}\) be the index set of \(\mathcal{D}\), split as \(\mathcal{I} = \mathcal{I}_{\text{train}} \sqcup \mathcal{I}_{\text{cal}}\), where \(\mathcal{I}_{\text{cal}} := \{1,\dots,n\}\).

The base model is first trained on \(\mathcal{D}_{\text{train}}\) to obtain a predictor \(\hat{\mu}\). Then we define a nonconformity score function \(s_{\hat{\mu}}: \mathcal{X} \times \mathcal{Y} \to \mathbb{R}\)  (e.g. \(s_{\hat{\mu}}(x,y) = |y - \hat{\mu}(x)|\) in regression) and compute the nonconformity scores on the calibration set $\mathcal{S}_{\text{cal}} = \{\, s_i := s_{\hat{\mu}(\mathcal{D}_{\text{train}})}(X_i, Y_i) : i =1,\dots,n\}$.

The SCP prediction set is
\begin{equation}
\label{split_CPset}
\mathcal{C}(X_{n+1})
= \bigl\{\, y \in \mathcal{Y} : s_{\hat{\mu}(\mathcal{D}_{\text{train}})}(X_{n+1}, y) \le s_{(k)} \,\bigr\},
\end{equation}
where \(s_{(\ell)}\) denotes the \(\ell\)-th order statistic of $\mathcal{S}_{\text{cal}}$ and \(k = \lceil (1-\alpha)(n+1) \rceil\)\footnote{We assume \(\alpha(n+1)\ge 1\),  so that $s_{(k)}$ is finite. If \(\alpha(n+1)<1\), the standard convention sets the quantile to \(+\infty\), yielding the trivial prediction set $\mathcal{C}(X_{n+1})=\mathbb 
R$ \citep{angelopoulos2023conformal}.}.

Under the assumption that \(\mathcal{D}_{\text{cal}} \cup (X_{n+1}, Y_{n+1})\) are exchangeable, the prediction set~\eqref{split_CPset} satisfies
\begin{equation}
\label{valid_coverage}
\mathbb{P}\big( Y_{n+1} \in \mathcal{C}(X_{n+1}) \big) \;\ge\; 1-\alpha.
\end{equation}

The definition of the prediction set in~\eqref{split_CPset} can be equivalently expressed in terms of 
p-values. Instead of comparing the test score directly to the empirical quantile $s_{(k)}$, 
one can compute the \textit{conformal p-value} for any candidate  $y \in \mathcal{Y}$ as
\begin{equation}
\label{pval_CP}
    P_{n}(y) = \frac{1 + \sum_{i=1}^n \mathbf{1}\!\big(s_i \geq s_{\hat{\mu}(\mathcal{D}_{\text{train}})}(X_{n+1}, y)\big)}{n+1},
\end{equation}

The corresponding prediction set is
\begin{equation}
\label{pval_CP_set}
    \mathcal{C}^{\textbf{pv}}(X_{n+1}) = \{\, y \in \mathcal{Y} : P_{n}(y) > \alpha \,\}.
\end{equation}
In  what follows, we assume the scores have a continuous joint distribution, avoiding ties (which is natural in the context of regression). Then, under exchangeability $P_n$ is uniformly distributed on 
$\{1/(n+1), \dots, 1\}$ under the null hypothesis 
\(\mathbb{H}_0: Y_{n+1}=y\)~\citep{vovk2005algorithmic,lei2018distribution}. 
Hence $P_n$ is a valid p-variable, i.e.,
\[
\mathbb{P}_{\mathbb{H}_0}\big(P_n(y) \le \delta \big) \le \delta, 
\quad \forall\, \delta \in (0,1).
\]
This formulation shows that the CP set consists exactly of the candidates $y$ 
for which the conformal p-value exceeds~$\alpha$. We will refer to this p-value formulation as p-CP.

\subsection{Conformal Prediction with e-values}
\label{subsec:e-CP}
Similarly to p-variables, \emph{e-variables} can be used to test the null hypothesis $\mathbb{H}_0$. E-variables were introduced by \citet{vovk2021values} and offer several advantages over p-variables \citep{vovk2021values,vovk2024merging,wang2025only}. The realizations of e-variables are referred to as \emph{e-values}.

\begin{definition}\citep{vovk2021values,evaluespoly}
    An \emph{e-variable} $E$ for a null hypothesis $\mathbb{H}$ is a random variable taking values in \([0,+\infty)\) satisfying $\mathbb{E}_{\mathbb{H}}[E]\leq 1$. It is said to be \textit{exact} if $\mathbb{E}_{\mathbb{H}}[E]= 1$. 
\end{definition}
Given an e-variable \(E\) with realized e-value \(e\), a level-\(\alpha\) test rejects $\mathbb{H}_0$ whenever \(e \geq 1/\alpha\). By Markov’s inequality (MI), we obtain the following Proposition. 

\begin{proposition}
\label{Prop:markov_coverage}
Let $E$ be an e-variable and $\alpha\in(0,1).$ Then $\mathbb{P}(E < 1/\alpha)  \geq 1 - \frac{\mathbb{E}_{\mathbb{H}}[E]}{1/\alpha} \ge 1 - \alpha$.
\end{proposition}

Therefore, replacing the P-variable \(P_n\) in~\eqref{pval_CP} with an E-variable \(E_n\) as the test statistic yields the conformal prediction set  
\[
    \mathcal{C}^{\textbf{ev}}(X_{n+1})
    = \bigl\{\, y \in \mathcal{Y} : E_n(y) < 1/\alpha \,\bigr\},
\]
which has valid finite-sample coverage as a direct consequence of Proposition~\ref{Prop:markov_coverage}. We will refer to this e-value formulation as e-CP.

\citet{balinsky2024} proposed such e-variable of the form  
\begin{equation}\label{balinsky_evalue}
E_{n}(y)
:= \frac{(n+1)\,s_{\hat\mu}(X_{n+1},y)}
        {\sum_{i=1}^n s_i + s_{\hat\mu}(X_{n+1},y)},
\quad y \in \mathcal{Y},
\end{equation}
which is exact, and can be directly employed within the CP framework. 

One advantage of e-CP is that it naturally connects CP to the broader theory of e-values, thereby inheriting many of their desirable properties. Recent work by \citet{gauthier2025values,gauthier2025backward}  demonstrates how e-CP generalizes classical CP in certain applications. However, an important open question remains: \textit{how efficient is e-CP compared to p-CP?}

For the remainder of the paper, we omit the explicit notation `$y \in \mathcal{Y}$' when writing prediction sets.

\subsection{Fundamental properties of e-values}
\label{subsec:evals_prop}
We briefly review several properties of e-values, mainly derived from \citet{vovk2021values}. These properties will later be leveraged in our e-CP methods.

\paragraph{E-Merging.} 

Merging e-variables is done via \emph{e-merging functions}, which essentially transform a vector of $K \geq 1$ e-variables \(\mathbf{E}=(E_1,\ldots,E_K)\), for any hypothesis, into a single e-variable \(M(\mathbf{E})\). For example, it is easy to check that the arithmetic mean is an e-merging function. An e-merging function $M$ is said to be \emph{admissible}, if there is no other e-merging function $N$ such that $N \geq M$ (pointwise) and $N(\mathbf{e})>M(\mathbf{e})$ for some $\textbf{e}$. Unlike p-values, e-values admit simple admissible merging rules that remain valid even under arbitrary dependence.\footnote{By contrast, the arithmetic mean of $K$ p-values $p_1, \ldots, p_K$ is generally not a p-value; at best one has $\mathbb{P}\!\Big(\tfrac{1}{K}\sum_{i=1}^K p_i \le \alpha\Big) \le 2\alpha,$ and the constant $2$ is sharp.}

\begin{proposition}
\label{prop:sym_merg}
The arithmetic mean essentially dominates any symmetric e-merging function. 
\end{proposition}

\begin{proposition}[\citet{wang2025only}]
\label{prop:Wang_merg}
The only admissible merging functions for arbitrary e-values are weighted arithmetic means.
\end{proposition}

Given a vector of $K$ e-variables $\mathbf{E}=(E_1,\ldots,E_K)$, the choice of e-merging function directly affects the efficiency of the resulting e-CP set. Indeed, if $M_1 \geq M_2$ are two e-merging functions, then for any $\alpha\in(0,1)$ we  have $\{M_1(\mathbf{E}) < 1/\alpha\} \subseteq \{M_2(\mathbf{E}) < 1/\alpha\}$, while both sets satisfy coverage.

\paragraph{Extensions of MIs.} 
The practicality of e-values stems in part from their connection to Markov's Inequality, which enables direct thresholding. We present here two useful extensions of MI when working with e-variables. 

\begin{proposition}[\citet{ramdas2023randomizedexchangeableimprovementsmarkovs}]
\label{prop:URMI}
Let $E$ be an e-variable, independent of $U \sim \mathrm{Unif}(0,1)$. Then,
\begin{equation*}
    \mathbb{P}\!\left(E > \tfrac{U}{\alpha}\right) \leq \alpha,
\end{equation*}
which is referred to as the \emph{Uniformly Randomized Markov's Inequality (UR-MI)}.
\end{proposition}

\begin{proposition}[\citet{Manole_2023_exch}]
\label{prop:UR-EMI}
Let $E_1, \dots, E_n$ be exchangeable e-variables, and let $U \sim \mathrm{Unif}(0,1)$ be independent. Then,
\begin{equation*}
    \mathbb{P}\!\left(\bigl\{E_1 \ge \tfrac{U}{\alpha}\bigr\} \cup 
    \biggl\{\sup_{k \leq n} \frac{1}{k}\sum_{i=1}^k E_i \ge \tfrac{1}{\alpha}\biggr\}\right) \leq \alpha.
\end{equation*}
\end{proposition}

\section{A new set-preserving P2E calibrator}\label{sec:ptoe_cp}

This section establishes the core technical foundation of the paper: a novel calibrator with useful properties from the conformal p-variable to an e-variable yielding tight prediction sets. We focus on CP problems involving the aggregation of multiple p-values.

E-variables can be constructed in different ways, but it is generally unclear how a given design choice affects the constructed prediction set. For example, we show in Appendix~\ref{app:e-variable_balinsky comparaison} that the construction in \eqref{balinsky_evalue} can lead to overly conservative prediction sets in the SCP framework. We therefore start from the conformal p-variable and calibrate it into an e-variable, building on the existing calibration literature.

The proofs of subsequent results are found in Appendix~\ref{proffs_appendix}. 

\subsection{Limitations of classical P-to-E calibrators}

\paragraph{P-to-e calibrators.} A standard tool for converting p-values into e-values is a \emph{p-to-e calibrator} \citep{vovk2021values,evaluespoly}.

\begin{definition}
A \emph{p-to-e calibrator} is a decreasing function $F:[0,1]\to[0,\infty]$ such that, for any hypothesis $\mathcal H$ and any p-variable $P$ for $\mathbb H$, $F(P)$ is an e-variable for $\mathbb H$.
\end{definition}
Examples of p-to-e calibrators include $F_1(p):=-\log(p)$, $F_2(p):=p^{-1/2}-1$, and $F_3(p):=2(1-p)$.
We also highlight the \emph{all-or-nothing} (AoN) calibrator (for a fixed $\alpha\in(0,1)$):
\[
F_{\mathrm{AoN}}(p)\;:=\;\frac{1}{\alpha}\,\mathbf{1}_{\{p\le \alpha\}}.
\]
Thus, given the conformal p-variable $P_n$ in \eqref{pval_CP}, $F_1(P_n),F_2(P_n),F_3(P_n)$ and $F_{\mathrm{AoN}}(P_n)$ are all e-variables for $\mathbb{H}_0$.
We can then use these e-variables for conformal inference, from which the $1-\alpha$ coverage guarantee is recovered simply by thresholding at $1/\alpha$ as explained in Section~\ref{subsec:e-CP}.



\paragraph{Set-Preservation and Efficiency.} Although coverage is immediate once we have an e-variable, the choice of the p-to-e calibrator $F$ can substantially impact the size of the resulting prediction sets. Since predicting this efficiency effect from the calibrator alone is difficult, as it also depends on the underlying score distribution, and the specific construction of the prediction set,  we focus on a structural property of the calibrator: \emph{set-preservation}. 

\begin{definition}

A p-to-e calibrator $F$ is said to be \emph{set-preserving} (at level $\alpha$) if, for all $n$ and for any nonconformity score used to build the p-variable $P_n$ in \eqref{pval_CP}, we have
\[
\mathcal{C}^{\textbf{pv}}(X_{n+1})
    = \{ P_n > \alpha \}
    = \{ E_{n} < 1/\alpha \}
    = \mathcal{C}^{\textbf{ev}}(X_{n+1}).
\]
where  $E_n:=F(P_n)$. It is said to be \textit{set-inflating} if  $$|\{ P_n > \alpha \}|\le 
|\{ E_{n} < 1/\alpha \}|.$$
\end{definition}

Set-preservation ensures that the conformal p-variable and its calibrated e-variable induce  the same SCP set, regardless of the calibration size and nonconformity score. Set-inflation means that the calibrated e-value set has size at least
as large as the standard SCP set, and can therefore be viewed as an
efficiency criterion. We characterize the class of calibrators (independent of the calibration set size $n$) satisfying this property:

\begin{proposition}\label{prop:only_aon_set_preserving}
Among all left-continuous p-to-e calibrators, only $F_{\mathrm{AoN}}$ is set-preserving.

\end{proposition}

To provide intuition and illustrate the significance of set-preservation, consider an aggregation context where we have $K$ conformal p-variables $\{P_n^{(k)}\}_{k=1}^K$ (e.g., from $K$ predictors, or $K$ folds) for the same hypothesis $\mathbb{H}_0$, and their calibrated e-variables $E_n^{(k)}:=F(P_n^{(k)}), k=1,\dots,K$.
A natural e-merging rule is to average the e-variables, yielding the prediction set:
\begin{equation}
    \label{agg_CP_set_F}
\mathcal{C}^{F}_{\mathbf{Agg}}(X_{n+1})
\;:=\;
\Big\{\frac{1}{K}\sum_{k=1}^K E_n^{(k)} < \frac{1}{\alpha} \Big\}.
\end{equation}
We denote the set $\mathcal{C}^{F}_{\mathbf{Agg}}(X_{n+1})$ to emphasize the dependence on $F$. Through union and intersection bounds, we obtain:
\[
\bigcap_{k=1}^K \{E_n^{(k)}<1/\alpha\}
\subseteq
\mathcal{C}^{F}_{\mathbf{Agg}}(X_{n+1})
\subseteq
\bigcup_{k=1}^K \{E_n^{(k)}<1/\alpha\}.
\]
When the calibrator $F$ is not set-preserving, the individual e-CP sets $\mathcal{C}_k^{\mathbf{ev}}(X_{n+1}):=\{E_n^{(k)} < 1/\alpha\}$ are larger than the SCP sets $\mathcal{C}_k^{\textbf{pv}}(X_{n+1}):=\{P_n^{(k)}>\alpha\}$. This expansion widens the resulting sandwich bound, and  explains why it may lead to overly conservative prediction sets. Conversely, if $F$ is set-preserving, the following sandwich bound holds:$$\bigcap_{k=1}^K \mathcal{C}_k^{\textbf{pv}}(X_{n+1})
\;\subseteq\;
\mathcal{C}^{F}_{\mathbf{Agg}}(X_{n+1})
\;\subseteq\;
\bigcup_{k=1}^K\mathcal{C}_k^{\textbf{pv}}(X_{n+1}),
$$
This justification gives intuition on why set-preserving calibrators can yield tight prediction sets. This behavior is further confirmed empirically in Section~\ref{sec:experiments}, where we observe that set-inflating p-to-e calibrators (like $F_1,F_2,F_3)$ result in significantly larger sets. 

Consequently, to achieve satisfying efficiency, Proposition \ref{prop:only_aon_set_preserving} indicates that the only p-to-e calibrator one should use is $F_{\mathrm{AoN}}$. However, the latter has some limitations as it can take zero values which destroys statistical evidence. Indeed, this can create a practical and theoretical drawback in sequential and optional-stopping contexts, where products of e-values $\prod_t E_t$ naturally arise: a single zero e-value collapses the entire product and prevents accumulation of evidence.
A similar issue can also arise in aggregation: under independence, the product is a well-defined e-merging function \citep{vovk2021values}, but again a single zero forces the merged e-value to zero\footnote{this  relates to the notion of e-power  \citep{evaluespoly}, defined as $\mathbb{E}_{\mathbb{H}}[\log E]$. The e-power of $F_{\mathrm{AoN}}$ is equal to $-\infty$. }. Since our framework can also support these broader uses of e-values, we seek an alternative.

\subsection{Our P2E Calibrator}

We propose a novel set-preserving \textbf{P2E calibrator}, which depends on both $n$ and $\alpha$, denoted by $F_{n,\alpha}$, that addresses the limitations discussed above. Specifically, it satisfies three conditions:

\textbf{A.} $F_{n,\alpha}$ is set-preserving,

\textbf{B.}  $E_{n,\alpha} := F_{n,\alpha}(P_n)$ is an exact e-variable\footnote{Exactness avoids unnecessary conservativeness: if \(\mathbb{E}[E]=c<1\), then the exact e-variable \(E/c\) yields the smaller set $\{E/c<1/\alpha\}\subseteq\{E<1/\alpha\},$ with both  satisfying coverage by MI.},

\textbf{C.} $F_{n,\alpha}$ is smooth, invertible, and strictly positive.



Let $\mathcal{E}$ denote the class of all positive, strictly decreasing real functions on $[0,1]$. For any $f \in \mathcal{E}$, the p-CP set in \eqref{pval_CP_set} can equivalently be written as
\begin{equation}
\label{equality_setsCP}
   \mathcal{C}^{\textbf{pv}}(X_{n+1}) 
   =  \left\{ f(P_n)<f(\alpha) \right\}=\left\{ \frac{f(P_n)}{\alpha f(\alpha)} < 1/\alpha \right\}. 
\end{equation}

Our key idea is to choose a suitable function $f \in \mathcal{E}$ such that the ratio $f(P_n)/\alpha f(\alpha)$ is itself an e-variable.  The following proposition asserts the existence of such a function.  

\begin{proposition}
\label{prop:fnalpha}
Let $n \geq 1$, and let $P_n$ be the p-variable in \eqref{pval_CP}. For $\alpha \in (0,1)$, suppose $\alpha(n+1) \in (1,\infty) \setminus \mathbb{N}$. Then, there exists a function $f_{n,\alpha} \in \mathcal{E}$ such that
\begin{equation}
\label{E(F)=1}
    \mathbb{E}_{\mathbb{H}_0}\!\left(\frac{f_{n,\alpha}(P_n)}{\alpha f_{n,\alpha}(\alpha)}\right) = 1.
\end{equation}

\end{proposition}


We denote by $\mathcal{E}_{n,\alpha}$ the set of all functions of the form $p\to f_{n,\alpha}(p)/\alpha f_{n,\alpha}(\alpha)$ of Proposition~\ref{prop:fnalpha}, that we call \textbf{P2E calibrators}. 
Recall that these functions differ from standard p-to-e calibrators, as they depend on both $\alpha$ and $n$. Additionally, we show that all P2E calibrators necessarily converge to $F_{\text{AoN}}$.
\begin{proposition}
\label{prop:convergence_AON}
   Let $n_\ell\to \infty$ be any sequence such that $\alpha(n_\ell+1)\notin \mathbb{N}$. Then, for any $F_{n_\ell}\in\mathcal{E}_{n_\ell,\alpha}$  we have $F_{n_\ell}(p)\xrightarrow{} F_{\mathrm{AoN}}(p) \quad \forall p\in(0,1).$
\end{proposition}
Thus, Proposition \ref{prop:convergence_AON} establishes a performance guarantee for P2E calibrators: \textit{Unlike standard p-to-e calibrators, they are set-preserving, and asymptotically recover the behavior of $F_{\mathrm{AoN}}$ which is itself set-preserving.}

We now present our main theorem, providing an explicit P2E calibrator which satisfies the three conditions \textbf{A-C}.  

\begin{theorem}
\label{main_theorem}
Let $P_n$ be the p-variable in \eqref{pval_CP} and suppose $\alpha(n+1) \in (1,\infty)\setminus \mathbb{N}$. Then, for any $s \in \left(\alpha, \frac{\lceil \alpha(n+1)\rceil}{n+1}\right)$, there exists $C_{n,\alpha} > 0$ such that 
\begin{equation}
\label{final_evalue}
F_{n,\alpha}(p) :
= \frac{1}{\alpha} \cdot 
\frac{1 + \exp\!\big(C_{n,\alpha}(\alpha - s)\big)}
     {1 + \exp\!\big(C_{n,\alpha}(p - s)\big)}
\end{equation}
is in the set $\mathcal{E}_{n,\alpha}$. Additionally, $F_{n,\alpha}\ge F_{\mathrm{AoN}}$ (pointwise).
\end{theorem}
We visualize in Figure \ref{f_n vizualization} the graph of $F_{n,\alpha}$ of Theorem \ref{main_theorem}.

The last domination property of the theorem is very relevant in our context. 
Indeed, when aggregating \(K\) p-variables $P_n^{(1)},...,P_n^{(K)}$,  it holds that 
$\frac{1}{K}\sum_{k=1}^K F_{n,\alpha}(P_n^{(k)}) \ge \frac{1}{K}\sum_{k=1}^K F_{\mathrm{AoN}}(P_n^{(k)}),$
which implies, following the notation of \eqref{agg_CP_set_F}, that
\begin{equation}
    \label{equation:FAon_larger}\mathcal{C}^{F_{n,\alpha}}_{\mathrm{agg}}(X_{n+1})\ \subseteq\ \mathcal{C}^{F_{\mathrm{AoN}}}_{\mathrm{agg}}(X_{n+1}).
\end{equation} Thus, our P2E calibrator \textit{is always more efficient than $F_{\mathrm{AoN}}$ }in this aggregation setting.

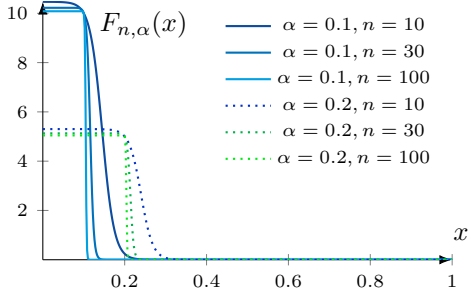
\begin{figure}
    \centering
    \begin{tikzpicture}
    \centering
  \begin{axis}[
    width=7cm, height=5cm,
    axis lines=middle,           
      axis line style={-Latex},
    xlabel={$x$},
    xlabel style={at={(axis description cs:0.98,0.1)},anchor=west},
ylabel={$F_{n,\alpha}(x)$},
ylabel style={xshift=4ex},tick label style={font=\scriptsize},
    legend style={
      at={(0.98,0.98)},      
      anchor=north east,     
      font=\scriptsize,   
      draw=none,             
      fill=none              
    }
  ]
  
    \addplot[cyan!20!blue, thick] table[x=x, y=n10] {sigmoid_alpha0.1.dat};
    \addlegendentry{$\alpha=0.1, n=10$}

    \addplot[cyan!50!blue, thick] table[x=x, y=n30] {sigmoid_alpha0.1.dat};
    \addlegendentry{$\alpha=0.1, n=30$}

    \addplot[cyan!90!blue, thick] table[x=x, y=n100] {sigmoid_alpha0.1.dat};
    \addlegendentry{$\alpha=0.1, n=100$}

    \addplot[green!20!blue, dotted, thick] table[x=x, y=n10] {sigmoid_alpha0.2.dat};
    \addlegendentry{$\alpha=0.2, n=10$}

    \addplot[green!70!blue, dotted, thick] table[x=x, y=n30] {sigmoid_alpha0.2.dat};
    \addlegendentry{$\alpha=0.2, n=30$}

    \addplot[green!92!blue, dotted, thick] table[x=x, y=n100] {sigmoid_alpha0.2.dat};
    \addlegendentry{$\alpha=0.2, n=100$}
  \end{axis}
\end{tikzpicture}
    \caption{Behavior of the mapping $x \mapsto F_{n,\alpha}(x)$ for varying values of $n$ and $\alpha$. }
        \label{f_n vizualization}
\end{figure}

For the remainder of this paper, we adopt the P2E calibrator of Theorem \ref{main_theorem}.


\section{Related Work}

CP \citep{saunders1999transduction,vovk2005algorithmic} is a distribution-free framework for uncertainty quantification; see \citet{Fontana_2023,angelopoulos2023conformal} for surveys. SCP \citep{papadopoulos2002inductive,lei2018distribution} requires only one model fit but can be inefficient, while Full-CP \citep{vovk2005algorithmic} is more efficient but computationally costly. Intermediate approaches such as CCP \citep{vovk2015cross,vovk2018cross}, multi-split \citep{solari2022multi}, jackknife+ \citep{barber2021jacknife}, and out-of-bag conformal methods \citep{gupta2022nested} aim to balance efficiency and computation by reusing data more effectively. CP has also been applied to aggregation, where the aim is to combine model uncertainties into a single prediction set. Some methods aggregate the prediction sets \citep{yang2024selectionaggregationconformalprediction,gasparin2024merginguncertaintysetsmajority}, while others operate on the nonconformity scores \citep{luo2025weightedaggregationconformityscores,rivera2025conformalpredictionensemblesimproving,SACP}.


CP with e-values is discussed in \citet{vovk2025econformal}, where it is shown to be conceptually simpler than p-CP. A key advantage is that e-values, which can be transformed into p-values and vice-versa, offer many beneficial properties for hypothesis testing \citep{evaluespoly,vovk2021values,vovk2024merging,wang2025only}.

In the CP framework, \citep{balinsky2024} independently design an e\mbox{-}variable based on exchangeable scores for traditional CP, while \citet{koningCP} derive a related more general class of e\mbox{-}values valid for conformal inference. Building on \citep{balinsky2024}, \citet{gauthier2025values,gauthier2025backward} extend CP to regimes where p-values do not fit; however, their methods can lack efficiency, which we believe is due to the e-variable design. Additionally, none of these works make an explicit link with classical p-CP. In our context, we  study the link between e-CP and p-CP through p-to-e calibrators.  Perhaps the closest approach to P2E is \citep{bickel2025small}, which gives sample\mbox{-}size\mbox{-}dependent mappings from p\mbox{-}values to Bayes factors and not in the CP context. 

\section{Our P2E Calibrator For Cross-Conformal Prediction and Aggregation}


\paragraph{E-Cross-Conformal Prediction.}

Cross-Conformal Prediction (CCP) \citep{vovk2015cross, vovk2018cross} strikes a balance between SCP, which can lack efficiency, and full CP, which is computationally intensive \citep{vovk2005algorithmic}. CCP can be viewed as a combination of SCP and cross-validation. The dataset \(\mathcal D\), indexed by \(\mathcal I\), is partitioned into
\(K\) equal folds \(\mathcal D_1,\ldots,\mathcal D_K\) with index sets
\(\mathcal I_1,\ldots,\mathcal I_K\), where
\[
\mathcal I=\bigsqcup_{k=1}^K \mathcal I_k,
\qquad |\mathcal I_k|=m=|\mathcal I|/K .
\] 

We assume $\alpha(m+1)\in(1,\infty)\setminus\mathbb N$. For each observation with index $i$, the CCP score is computed using a model trained without the fold containing $i$: $s_i = s_{\hat\mu(\mathcal{D}\setminus \mathcal{D}_{k(i)})}(X_i, Y_i), 
\quad i=1,\dots,n$, where $\mathcal{D}_{k(i)}$ denotes the fold containing observation $i$. We assume the scores have a continuous distribution.

As formulated in \citet{gasparinimproving}, the CCP set is defined as
\begin{equation}
\label{ccp_set}
\mathcal{C}^{ccp}(X_{n+1}) 
:= \Bigl\{\, y : \frac{1}{K} \sum_{k=1}^{K} P_k(y) 
> \alpha + (1-\alpha)\tfrac{K-1}{K+|\mathcal{I}|} \Bigr\},
\end{equation}
where the fold-wise conformal p-values are given by
\begin{equation}
\label{cross_pval}
P_k(y) := \frac{1+\sum_{i\in\mathcal{I}_k} 
\mathbf{1}\!\left\{ s_{\hat\mu(\mathcal{D}\setminus \mathcal{D}_{k(i)})}(X_{n+1},y)\le s_i \right\}}{m+1}.
\end{equation}

The CCP set in \eqref{ccp_set} satisfies the finite-sample coverage guarantee
\begin{equation}
\mathbb{P}\!\big(Y_{n+1} \in \mathcal{C}^{ccp}(X_{n+1})\big) 
\ge 1 - 2\alpha - 2\frac{(1-\alpha)(1-1/K)}{|\mathcal{I}|/K + 1}. \label{eq:ccpbound}
\end{equation}

A key limitation of CCP is its weaker marginal coverage guarantee, even though in practice it often attains empirical coverage close to $1-\alpha$.



\citet{gasparinimproving} further improve the efficiency of CCP by exploiting the exchangeability of fold-wise p-values to leverage stronger aggregation results \citep{Gasparin_2025_exch}. All of their proposed variants come with  a $1-2\alpha$ coverage guarantee. 


Our contribution builds on CCP and its extensions by introducing an e-value–based CCP method that remains valid under arbitrary dependence among the fold-wise p-values, satisfies the $1-\alpha$ finite-sample coverage guarantee, and preserves the computational complexity of standard CCP. The main idea is to transform the fold-wise p-variables $P_k$ \eqref{cross_pval} into e-variables using our P2E calibrator:
\begin{equation}
\label{eval_CCP}
    E_k^{ccp}(y) = F_{m,\alpha}(P_k(y)),\; k=1,\dots,K,
\end{equation}
Then, we admissibly merge them by averaging and apply Proposition~\ref{prop:URMI} to obtain
\begin{equation}
\label{ECCP}
    \mathcal{C}_{ccp}^{\textbf{ev}}(X_{n+1}) = \left\{ \frac{1}{K}\sum_{k=1}^K E_k^{ccp}<U/{\alpha}\right\}.
\end{equation}
We refer to this method as \textbf{ECCP}. In the special case where the \(P_k\), and consequently the \(E_k\), are exchangeable, we further propose two variants that leverage Proposition~\ref{prop:UR-EMI}:
\begin{equation}
\label{Ex_ECCP}
   \mathcal{C}_{ccp}^{\textbf{ev}-ex}(X_{n+1}) = \left\{ \sup_{t\le K}\frac{1}{t}\sum_{k=1}^t E_k^{ccp}<{1}/{\alpha}\right\},
\end{equation}
\begin{equation}
\label{Ex_ECCP_U}
   \mathcal{C}_{ccp,U}^{\textbf{ev}-ex}(X_{n+1})=\mathcal{C}_{ccp}^{\textbf{ev}-ex}(X_{n+1})\cap \left\{ E_1<{U}/{\alpha}\right\},
\end{equation}
with $U\sim \text{Unif}(0,1)$ an independent variable. We refer to these methods as \textbf{ECCP-Exch} and \textbf{UR-ECCP-Exch}. 

\begin{proposition}
\label{prop:ECCP}If $(X_{n+1},Y_{n+1})$ is exchangeable with $\mathcal{D}$,
     the e-CCP set \eqref{ECCP} satisfies the  finite-sample coverage guarantee in \eqref{valid_coverage}.  If, in addition, the fold-wise e-values
\(E^{\mathrm{ccp}}_1,\ldots,E^{\mathrm{ccp}}_K\) are exchangeable, then the sets \eqref{Ex_ECCP} and \eqref{Ex_ECCP_U} satisfy \eqref{valid_coverage}.
\end{proposition}


\paragraph{E-Conformal Aggregation.}
\label{E-CA}
Conformal Aggregation (CA) combines nonconformity scores or the resulting prediction sets from multiple base models into a single prediction set, while retaining finite-sample validity.  
The e-value framework is thus natural in aggregation problems,  because we can use admissible e-merging rules regardless of the dependence of the e-variables. 

We adopt the SCP framework with $K>1$ calibration sets $\mathcal{D}_{cal}^{(1)},\dots,\mathcal{D}_{cal}^{(K)}$ of sizes $m_1,\dots,m_K$, each associated with a pretrained predictor $\hat{\mu}^1,\dots,\hat{\mu}^K$ and corresponding nonconformity scores $s_1,\dots,s_K$, assumed to have continuous distribution. We assume $\alpha(m_k+1)\in (1,\infty)\setminus\mathbb{N}, \forall k$. For a new $X_{n+1}$, we compute the p-variables:
\begin{equation}
\label{pval_agg}
    P_k(y) = \frac{1}{m_k+1}\Big(1 + \sum_{z\in \mathcal{D}_{cal}^{(k)}} \mathbf{1}\{s_k(z) \ge s_k(X_{n+1},y)\}\Big), 
\end{equation}
and obtain e-variables via our P2E calibrator:
\begin{equation}
\label{eval_agg}
    E_k(y) = F_{m_k,\alpha}\!\big(P_k(y)\big), \quad k=1,\dots,K.
\end{equation}

Let \(U \sim \mathrm{Unif}(0,1)\) be an independent random variable. By admissibly merging the e-values in \eqref{eval_agg}, we proceed analogously to the CCP construction. To achieve stronger performance, as indicated by \citet{wang2025only} result, we employ weighted aggregation of e-values:
\begin{align}
\label{WECA}
\mathcal{C}_{\text{Wagg}}^{\textbf{ev}}(X_{n+1}) 
   &= \Bigl\{\,   \sum_{k=1}^K \omega_k E_k < \tfrac{1}{\alpha} \,\Bigr\},
\end{align}
where the weights are selected to minimize some criterion,
and its randomized version 
$\mathcal{C}_{\text{Wagg}}^{\textbf{ev-U}}(X_{n+1})$,
obtained by replacing the threshold $1/\alpha$ with $U/\alpha$. 

To do so, we split each calibration set into two disjoint parts,
$\mathcal S_{\mathrm{cal}}^{(k)}
=
\mathcal S_{\mathrm{tune}}^{(k)}
\sqcup
\mathcal S_{\mathrm{inf}}^{(k)}.$
The tuning split $\{\mathcal D_{\mathrm{tune}}^{(k)}\}_{k=1}^K$ is used only to select the
aggregation weights
$\omega^\star=(\omega_1^\star,\ldots,\omega_K^\star)\in\Delta_K,$
whereas the inference split $\{\mathcal D_{\mathrm{inf}}^{(k)}\}_{k=1}^K$ is used only to
construct the final conformal e-values. 

\begin{proposition}
Assuming that, for each $k$,
$(X_{n+1},Y_{n+1})$ is independent of $\mathcal D_{\mathrm{tune}}^{(k)}$  and  exchangeable with $\mathcal D_{\mathrm{inf}}^{(k)}$, the sets $\mathcal{C}_{\text{Wagg}}^{\textbf{ev}}(X_{n+1}), \mathcal{C}_{\text{Wagg}}^{\textbf{ev-U}}(X_{n+1}) $ satisfy \eqref{valid_coverage}.
\end{proposition}

\begin{figure*}
    \centering
    \includegraphics[width=0.98\linewidth]{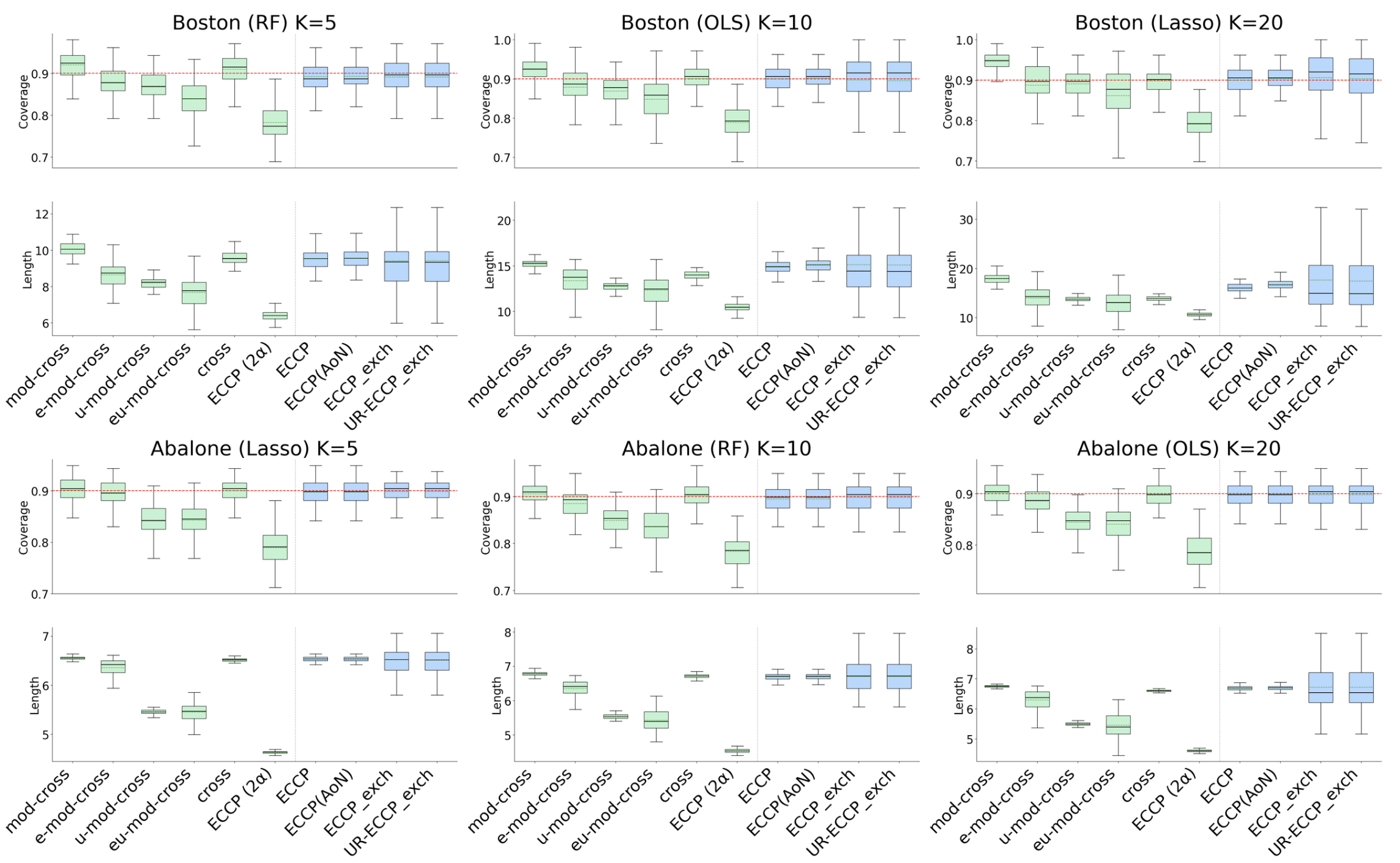}
    \caption{Distribution of mean prediction set length and empirical coverage obtained using different regression algorithms and different Cross-CP methods over 100 seeds, for $\alpha=0.1$}
    \label{fig:CCP_results}
\end{figure*}

We refer to these methods as \textbf{WECA}, and \textbf{UR-WECA}, respectively.
The core advantage of the WECA method lies in its ability to select \emph{data-dependent} weights that are optimized to minimize some type of criterion. In our case, we select the weights to minimize the average prediction set length. We describe the practical implementation of this method and how  coverage holds in Appendix~\ref{app:weca_validity}.

\section{Experiments and Results}
\label{sec:experiments}

We benchmark our methods on OpenML regression datasets \citep{vanschoren2014openml} using an extensive experimental suite that spans a diverse range of base models and multiple CCP configurations\footnote{\url{https://github.com/Nabil-Ala/P2E_calibration}}.

\subsection{Cross-Conformal Prediction}

We follow the experimental design of \citet{gasparinimproving}, where CCP p-values are constructed to be exchangeable, using three regressors: Ordinary Least Squares (OLS), Lasso Regression, and Random Forest (RF). 

\begin{table}[t]
\centering
\small
\setlength{\tabcolsep}{6pt}
\begin{tabular}{lcc}
\toprule
Method & Len. & Cov.  \\
\midrule
\multicolumn{3}{l}{\textbf{Abalone} ($K{=}10$, RF)}\\
ECCP        & $6.70 \pm 0.09$ & $0.90 \pm 0.02$ \\
ECCP($F_1$) & $29.07 \pm 1.93$ & $0.89 \pm 0.03$ \\
ECCP($F_2$) & $10.21 \pm 0.35$ & $0.90 \pm 0.02$ \\
ECCP($F_3$) & $32.02 \pm 1.99$ & $0.89 \pm 0.02$ \\
\midrule
\multicolumn{3}{l}{\textbf{Boston} ($K{=}10$, RF)}\\
ECCP        & $9.76 \pm 0.63$ & $0.89 \pm 0.03$ \\
ECCP($F_1$) & $67.15 \pm 3.99$ & $0.90 \pm 0.03$ \\
ECCP($F_2$) & $53.74 \pm 4.21$ & $0.92 \pm 0.02$ \\
ECCP($F_3$) & $59.18 \pm 4.17$ & $0.90 \pm 0.03$ \\
\bottomrule
\end{tabular}
\caption{CCP results with other calibrators. Values are reported as mean $\pm$ standard deviation across 100 random seeds.}
\label{tab:rf_abalone_boston_k10_other_ptoe}
\end{table}

\paragraph{Methods.}
CCP methods provide different coverage guarantees, so for a fair comparison we group methods having the same finite-sample coverage guarantee:

\begin{table*}
\centering
\caption{Empirical coverage and prediction set size for conformal aggregation methods. Values are reported as mean $\pm$ standard deviation across  $20$ random seeds, for $\alpha=0.05$.}
\label{table:e-CA_results}
\resizebox{\textwidth}{!}{
\begin{tabular}{l cccccccc}
\toprule
Method & \multicolumn{2}{c}{361234} & \multicolumn{2}{c}{361235} & \multicolumn{2}{c}{361237} & \multicolumn{2}{c}{361244} \\
\cmidrule(lr){2-3} \cmidrule(lr){4-5} \cmidrule(lr){6-7} \cmidrule(lr){8-9}
 & Cov. & Len. & Cov. & Len. & Cov. & Len. & Cov. & Size \\
\midrule
CM & 0.98 $\pm$ 0.01 & 3.85 $\pm$ 0.20 & 0.98 $\pm$ 0.01 & 3.27 $\pm$ 0.15 & 0.99 $\pm$ 0.01 & 2.88 $\pm$ 0.12 & 0.98 $\pm$ 0.01 & 6.90 $\pm$ 1.34 \\
CR & 0.97 $\pm$ 0.01 & 3.56 $\pm$ 0.19 & 0.97 $\pm$ 0.01 & 2.44 $\pm$ 0.17 & 0.96 $\pm$ 0.02 & 1.98 $\pm$ 0.15 & 0.98 $\pm$ 0.01 & 6.58 $\pm$ 1.28 \\
$P_{\mathrm{Agg}}$ & 0.98 $\pm$ 0.01 & 3.91 $\pm$ 0.20 & 0.99 $\pm$ 0.01 & 3.33 $\pm$ 0.12 & 0.99 $\pm$ 0.01 & 2.97 $\pm$ 0.10 & 0.98 $\pm$ 0.01 & 6.91 $\pm$ 1.32 \\
COLA-S & 0.95 $\pm$ 0.01 & 2.90 $\pm$ 0.19 & 0.96 $\pm$ 0.02 & 1.71 $\pm$ 0.32 & 0.95 $\pm$ 0.03 & 1.64 $\pm$ 0.25 & 0.97 $\pm$ 0.02 & 5.23 $\pm$ 1.48 \\
WECA & 0.95 $\pm$ 0.01 & 2.88 $\pm$ 0.19 & 0.95 $\pm$ 0.01 & 1.71 $\pm$ 0.29 & 0.95 $\pm$ 0.03 & 1.65 $\pm$ 0.25 & 0.96 $\pm$ 0.02 & 5.09 $\pm$ 2.06 \\
UR-WECA & 0.95 $\pm$ 0.01 & 2.88 $\pm$ 0.18 & 0.95 $\pm$ 0.01 & 1.70 $\pm$ 0.30 & 0.95 $\pm$ 0.03 & 1.64 $\pm$ 0.25 & 0.96 $\pm$ 0.02 & 4.93 $\pm$ 1.86 \\
\bottomrule
\end{tabular}
}
\end{table*}

\begin{table}[t]
\centering
\small
\setlength{\tabcolsep}{4pt}
\begin{tabular}{llcc}
\toprule
& Method & Cov.  & Len.  \\
\midrule
\multirow{8}{*}{\rotatebox[origin=c]{90}{\textbf{361234}}} & WECA($F_1$) & $1.00 \pm 0.00$ & $72.14 \pm 15.33$ \\
 & UR-WECA($F_1$) & $0.95 \pm 0.01$ & $50.36 \pm 10.76$ \\
 & WECA($F_2$) & $1.00 \pm 0.00$ & $7.75 \pm 1.40$ \\
 & UR-WECA($F_2$) & $0.95 \pm 0.01$ & $4.88 \pm 0.43$ \\
 & WECA($F_3$) & $1.00 \pm 0.00$ & $72.14 \pm 15.33$ \\
 & UR-WECA($F_3$) & $0.95 \pm 0.01$ & $65.34 \pm 13.92$ \\

\midrule
\multirow{8}{*}{\rotatebox[origin=c]{90}{\textbf{361235}}} & WECA($F_1$) & $1.00 \pm 0.00$ & $38.43 \pm 2.99$ \\
 & UR-WECA($F_1$) & $0.95 \pm 0.01$ & $28.73 \pm 2.48$ \\
 & WECA($F_2$) & $1.00 \pm 0.00$ & $38.43 \pm 2.99$ \\
 & UR-WECA($F_2$) & $0.96 \pm 0.02$ & $13.98 \pm 1.68$ \\
 & WECA($F_3$) & $1.00 \pm 0.00$ & $38.43 \pm 2.99$ \\
 & UR-WECA($F_3$) & $0.95 \pm 0.01$ & $34.76 \pm 2.93$ \\

\midrule
\multirow{8}{*}{\rotatebox[origin=c]{90}{\textbf{361237}}} & WECA($F_1$) & $1.00 \pm 0.00$ & $32.05 \pm 3.24$ \\
 & UR-WECA($F_1$) & $0.95 \pm 0.02$ & $24.65 \pm 2.71$ \\
 & WECA($F_2$) & $1.00 \pm 0.00$ & $32.05 \pm 3.24$ \\
 & UR-WECA($F_2$) & $0.96 \pm 0.01$ & $15.64 \pm 2.07$ \\
 & WECA($F_3$) & $1.00 \pm 0.00$ & $32.05 \pm 3.24$ \\
 & UR-WECA($F_3$) & $0.95 \pm 0.02$ & $28.89 \pm 2.91$ \\

\midrule
\multirow{8}{*}{\rotatebox[origin=c]{90}{\textbf{361244}}} & WECA($F_1$) & $1.00 \pm 0.00$ & $115.49 \pm 34.62$ \\
 & UR-WECA($F_1$) & $0.97 \pm 0.02$ & $87.96 \pm 25.82$ \\
 & WECA($F_2$) & $1.00 \pm 0.00$ & $115.49 \pm 34.62$ \\
 & UR-WECA($F_2$) & $0.97 \pm 0.02$ & $54.66 \pm 17.04$ \\
 & WECA($F_3$) & $1.00 \pm 0.00$ & $115.49 \pm 34.62$ \\
 & UR-WECA($F_3$) & $0.97 \pm 0.02$ & $103.89 \pm 30.73$ \\

\bottomrule
\end{tabular}
\caption{Performance of WECA and UR-WECA using alternative p-to-e calibrators: $F_1(p):=-\log(p)$, $F_2(p):=p^{-1/2}-1$, $F_3(p):=2(1-p)$.}
\label{tab:weca_rotated}
\end{table}

\textbf{$1-2\alpha$ coverage methods:}  the CCP baseline of \citet{vovk2015cross}, the variants of \citet{gasparinimproving} (\texttt{mod-cross}, \texttt{e-mod-cross}, \texttt{u-mod-cross}, \texttt{eu-mod-cross}), and our ECCP run at level $2\alpha$, denoted $\mathrm{ECCP}(2\alpha)$. 

\textbf{ $1-\alpha$ coverage methods:} our ECCP, ECCP-Exch, and UR-ECCP-Exch.
To highlight the role of our P2E calibrator, we also run $\mathrm{ECCP}$ with alternative p-to-e calibrators: $F_1(p)=-\log p$, $F_2(p)=p^{-1/2}-1$, $F_3(p):=2(1-p)$ and $F_{\mathrm{AoN}}$. 

We denote these variants by $\mathrm{ECCP}(F_1)$, $\mathrm{ECCP}(F_2)$, $\mathrm{ECCP}(F_3)$, and $\mathrm{ECCP}({\mathrm{AoN}})$.

\paragraph{Results.} We report the results as boxplots in Figure~\ref{fig:CCP_results}. Among the $1-2\alpha$ methods (shown in green), $\mathrm{ECCP}(2\alpha)$ is consistently the most efficient: it yields substantially smaller prediction sets compared to CCP and the variants of \citet{gasparinimproving}, while achieving empirical coverage close to the target coverage of $80\%$. 

Turning to the $1-\alpha$ methods (in blue), $\mathrm{ECCP}(\mathrm{AoN})$ shows great efficiency, as expected, since $F_{\mathrm{AoN}}$ is set-preserving. This supports the intuition behind our framework. Although we know theoretically that it induces sets larger than  ECCP (equation \eqref{equation:FAon_larger}), it asymptotically approaches ECCP's performance as the calibration set size increases (equivalently, as K decreases) as expected from Proposition \ref{prop:convergence_AON}.    The variants ECCP-Exch and UR-ECCP-Exch  also perform well on average, but exhibit higher variability across seeds, likely due to the asymmetric nature of their aggregation rule as pointed out by \citet{gasparinimproving}.

Comparing to classic p-to-e calibrators of the literature, $\mathrm{ECCP}(F_1)$, $\mathrm{ECCP}(F_2)$, and $\mathrm{ECCP}(F_3)$, perform substantially worse\footnote{Average length is computed on a finite bounded evaluation grid. For  $F_1,F_2,F_3$, the true conformal set may be unbounded (if the aggregated e-variable is always less than $1/\alpha$).}; we therefore report them separately in Table~\ref{tab:rf_abalone_boston_k10_other_ptoe}. These results align again with the idea that set-inflating calibrators can lead to noticeably weaker efficiency. Additional experiments are provided in Appendix \ref{app:additional_experiments}.

Finally, we also demonstrate in Appendix \ref{app:ECCP_RF} how our ECCP can improve one aspect of stability of standard CCP.

\subsection{Conformal Aggregation}


We consider the SCP framework for aggregation in section \ref{E-CA}, where we consider $K=7$ base learners including linear models, tree-based methods, neural networks, and Bayesian regressors.   
A key advantage of our framework is that WECA can accommodate settings in which different models are calibrated on calibration sets of different sizes. This is more realistic in applications where models may be trained or calibrated using data from different sources. To reflect this setting, we consider heterogeneous calibration sizes by assigning to each predictor a random subset of a larger calibration set; see Appendix~\ref{app:agg_random_splits} for details.

\paragraph{Methods.} We compare our proposed methods, \textbf{WECA} and \textbf{UR-WECA} based on our P2E calibrator, against  state-of-the-art aggregation approaches:  
majority vote (\textbf{CM}) and  its randomized variant (\textbf{CR}) \citep{gasparin2024merginguncertaintysetsmajority};  a p-value–based approach (\textbf{$P_{\text{Agg}}$}), which averages the conformal p-values per model and applies the threshold at $\alpha/2$ instead of $\alpha$. 

Additionally, we implement the recent aggregation method \textbf{COLA-s} from \citep{COLA}, which constructs prediction sets by intersecting conformal sets obtained at optimized miscoverage levels. All the above methods have the $1-\alpha$ coverage guarantee.

\textbf{Results.} The results are reported in Table~\ref{table:e-CA_results}. We observe that all methods achieve the finite-sample coverage guarantee, consistent with our theory. CM, CR and $P_{\text{Agg}}$  often over-cover, reflecting their conservative behavior. Our methods are the most efficient among the aggregation methods, highlighting the benefit of the randomization step with $U\sim\mathrm{Unif}(0,1)$  in the UR-WECA method, which consistently produces tighter prediction sets. The role of set-preservation is further supported by Table~\ref{tab:weca_rotated}, where we evaluate WECA and UR-WECA under the same setup but using alternative calibrators. Again, using the calibrators $F_1,F_2,F_3$ produce overly conservative prediction sets. 

Overall, the results suggest that the gains stem from combining our P2E calibration with admissible e-value merging, further enhanced by randomization.

 More broadly, our e-CP framework is modular and can be extended to additional settings and applications; we discuss   possible future directions in Appendix \ref{app:future_directions}.

\section{Conclusion}

We present an e-value perspective on CP by calibrating the conformal p-variable into an e-variable. We study how p-to-e calibration affects the size of conformal prediction sets, identifying set-preservation and set-inflation as key structural properties for understanding efficiency in the context of aggregating p-values. Building on this insight, we introduce a principled new P2E calibrator that is set-preserving. This construction unlocks the benefits of e-value theory and randomization results that we apply to enhance CA and CCP, while automatically retaining the desired \(1-\alpha\) coverage guarantee. Our experiments support the set-preserving intuition and demonstrate consistent improvements in efficiency. More broadly, our framework provides a practical step toward efficient and robust e-value-based distribution-free uncertainty quantification.


\section*{Impact Statement}

This work advances the field of reliable uncertainty quantification and machine learning. By providing mathematically grounded frameworks that enhance trust and transparency, our approach facilitates the deployment of practical, accessible models tailored for real-world applications.
Although our framework is domain-agnostic, it should not be applied in settings that enable harmful, unethical, or discriminatory uses.


\bibliographystyle{icml2026}
\bibliography{references}
\newpage
\appendix
\onecolumn

\section{Comparison of e-CP from \citep{balinsky2024} and p-CP sets}
\label{app:e-variable_balinsky comparaison}


Many p-CP methods yield sharp prediction sets, whereas e-CP methods offer greater flexibility and stronger theoretical guarantees through e-value theory. Our objective is to preserve the practical efficiency of p-CP while harnessing the theoretical advantages of e-CP. To evaluate this trade-off, we compare the resulting prediction set sizes, providing insight into the efficiency gap between the two approaches.

Recall that SCP operates on a set of calibration scores $s_1,\cdots,s_n$  and a new test point $X_{n+1}$, under the standard CP exchangeability assumption. Assuming the scores are  strictly positive, \citet{balinsky2024} proposed an e-variable defined as
\begin{equation}\label{appendix:balinsky_evalue}
E_{n}(y)
= \frac{(n+1)\,s_{\hat\mu}(X_{n+1},y)}{\sum_{i=1}^n s_i+ s_{\hat\mu}(X_{n+1},y)\,},
\qquad y \in \mathcal{Y}.
\end{equation}

So the corresponding e-CP set is given by
\begin{equation}\label{appendix:e-cpset}
\mathcal{C}^{\mathrm{ev}}(X_{n+1})
= \bigl\{ y \in \mathcal{Y} : E_n(y) < 1/\alpha \bigr\},
\end{equation}
and it enjoys finite-sample marginal coverage by Proposition~\ref{Prop:markov_coverage}. We will refer to this set as the e-CP(Balinsky).


This design offers two practical advantages:
\begin{enumerate}

\item Unlike p-CP, which compares $s_{\hat\mu}(X_{n+1},y)$ to the empirical $(1-\alpha)$ quantile $s_{(\lceil(1-\alpha)(n+1)\rceil)}$, e-CP requires no sorting at prediction time. We simply evaluate the e-variable for each candidate label.

\item It provides a principled choice of e-variable for CP, which can be readily applied or adapted to various settings (see, e.g., \citet{gauthier2025values}).

\end{enumerate}


Let us compare the prediction set \eqref{appendix:e-cpset} and the classic p-CP set. Denoting $S=\sum_{i=1}^n s_i$, we can rewrite \eqref{appendix:e-cpset} as
\begin{equation}\label{appendix:e-CP_set_rewrite}
\mathcal{C}^{\mathrm{ev}}(X_{n+1})=
 \Bigl\{\, y\in\mathcal{Y}: \; s_{\hat\mu}(X_{n+1},y) \;<\;
\frac{S}{\,\alpha(n+1)-1\,} \Bigr\}.
\end{equation}

Recall that the corresponding p-CP set is given by
\begin{equation}
\label{appendix:split_CPset}
\mathcal{C}^{\mathbf{pv}}(X_{n+1})
= \bigl\{\, y \in \mathcal{Y} : s_{\hat{\mu}}(X_{n+1}, y) \le s_{(k)} \,\bigr\},
\end{equation}
where \(k = \lceil (1-\alpha)(n+1) \rceil\).



We note that
\[
k=\lceil (1-\alpha)(n+1)\rceil \;\Rightarrow\; k-1 < (1-\alpha)(n+1),
\]
which implies
\[
n-k+1 \;>\; n-(1-\alpha)(n+1) \;=\; \alpha(n+1)-1.
\]
Since at least the top \(n-k+1\) scores are each \(\ge s_{(k)}\), it follows that
\[
S \;\ge\; \sum_{i=k}^n s_{(i)} \;\ge\; (n-k+1)\,s_{(k)}
\;>\; \bigl(\alpha(n+1)-1\bigr)\,s_{(k)}.
\]
Dividing by \(\alpha(n+1)-1\) (which is assumed to be $>0$) yields
\[
\frac{S}{\alpha(n+1)-1} \;>\; s_{(k)} \quad  \Longrightarrow\quad
\mathcal{C}^{\mathbf{pv}}(X_{n+1})\subseteq \mathcal{C}^{\mathbf{ev}}(X_{n+1}).
\]

Therefore, the p-CP set is always contained within the e-CP set for any target coverage level. In practice, this difference can be substantial. This observation underscores the main limitation of the e-variable defined in~\eqref{appendix:balinsky_evalue}, which tends to produce larger prediction sets. In contrast, our proposed e-variable in~\eqref{final_evalue} retains the sharpness of the p-CP construction while operating within the e-value framework.



To illustrate this point, we conduct a small experiment comparing the prediction sets in \eqref{appendix:e-CP_set_rewrite} and \eqref{appendix:split_CPset} on both regression and classification datasets from OpenML. We consider three base models: Linear Regression, Random Forest (RF), and a Multi-Layer Perceptron (MLP). Each dataset is split into 60\% for training, 20\% for calibration, and 20\% for testing. We report the mean prediction set length and empirical coverage in Tables \ref{appendix:table_regression} and \ref{appendix:table_classification}. We observe that p-CP achieves coverage close to the nominal level, whereas e-CP(Balinsky) consistently over-covers and produces substantially larger prediction sets.

\begin{table}[h]
\centering
\begin{tabular}{llcccc}
\toprule
\multirow{2}{*}{Dataset} & \multirow{2}{*}{Model} & \multicolumn{2}{c}{P-CP} & \multicolumn{2}{c}{E-CP(Balinsky)} \\
\cmidrule(lr){3-4}\cmidrule(lr){5-6}
 &  & Cov & Len & Cov & Len \\
\midrule
boston     & Linear & $0.884 \pm 0.055$ & $14.495 \pm 3.083$ & $1.000 \pm 0.000$ & $74.592 \pm 7.153$ \\
boston     & MLP    & $0.909 \pm 0.037$ & $12.678 \pm 1.568$ & $1.000 \pm 0.000$ & $66.596 \pm 7.998$ \\
boston     & RF     & $0.914 \pm 0.020$ & $10.613 \pm 0.805$ & $0.997 \pm 0.006$ & $50.816 \pm 3.965$ \\
\addlinespace
california & Linear & $0.901 \pm 0.005$ & $2.194 \pm 0.043$  & $1.000 \pm 0.000$ & $10.717 \pm 0.103$ \\
california & MLP    & $0.901 \pm 0.007$ & $1.613 \pm 0.034$  & $1.000 \pm 0.000$ & $7.282 \pm 0.166$ \\
california & RF     & $0.899 \pm 0.007$ & $1.562 \pm 0.051$  & $1.000 \pm 0.000$ & $6.746 \pm 0.153$ \\
\addlinespace
kin8nm     & Linear & $0.904 \pm 0.011$ & $0.651 \pm 0.012$  & $1.000 \pm 0.000$ & $3.283 \pm 0.063$ \\
kin8nm     & MLP    & $0.908 \pm 0.004$ & $0.269 \pm 0.006$  & $1.000 \pm 0.000$ & $1.284 \pm 0.025$ \\
kin8nm     & RF     & $0.904 \pm 0.016$ & $0.478 \pm 0.014$  & $1.000 \pm 0.000$ & $2.330 \pm 0.059$ \\
\bottomrule
\end{tabular}
\caption{Empirical mean prediction set size and coverage for 3 regression datasets (across 10 seeds, $\alpha=0.1$).}
\label{appendix:table_regression}
\end{table}

\begin{table}[h]
\centering
\begin{tabular}{llcccc}
\toprule
\multirow{2}{*}{Dataset} & \multirow{2}{*}{Model} & \multicolumn{2}{c}{P-CP} & \multicolumn{2}{c}{E-CP(Balinsky)} \\
\cmidrule(lr){3-4}\cmidrule(lr){5-6}
 &  & Cov & Len & Cov & Len \\
\midrule
breast\_cancer & Linear & $0.924 \pm 0.030$ & $0.93 \pm 0.03$ & $0.986 \pm 0.010$ & $1.01 \pm 0.02$ \\
breast\_cancer & MLP    & $0.911 \pm 0.037$ & $0.94 \pm 0.05$ & $1.000 \pm 0.000$ & $1.91 \pm 0.26$ \\
breast\_cancer & RF     & $0.916 \pm 0.035$ & $0.94 \pm 0.04$ & $0.999 \pm 0.003$ & $1.51 \pm 0.40$ \\
\addlinespace
cifar10       & Linear & $0.896 \pm 0.013$ & $7.27 \pm 0.14$ & $1.000 \pm 0.000$ & $10.00 \pm 0.00$ \\
cifar10       & MLP    & $0.902 \pm 0.010$ & $5.75 \pm 0.19$ & $1.000 \pm 0.000$ & $10.00 \pm 0.00$ \\
cifar10       & RF     & $0.893 \pm 0.012$ & $5.09 \pm 0.14$ & $1.000 \pm 0.000$ & $10.00 \pm 0.00$ \\
\addlinespace
mnist\_784    & Linear & $0.906 \pm 0.005$ & $1.07 \pm 0.03$ & $1.000 \pm 0.000$ & $10.00 \pm 0.00$ \\
mnist\_784    & MLP    & $0.903 \pm 0.008$ & $0.97 \pm 0.01$ & $1.000 \pm 0.000$ & $10.00 \pm 0.00$ \\
mnist\_784    & RF     & $0.900 \pm 0.008$ & $0.95 \pm 0.01$ & $1.000 \pm 0.000$ & $10.00 \pm 0.00$ \\
\bottomrule
\end{tabular}
\caption{Empirical mean prediction set size and coverage for 3 classification datasets (across 10 seeds, $\alpha=0.1$). }
\label{appendix:table_classification}
\end{table}

The following proposition characterizes the asymptotic ratio between the lengths of the e-CP(Balinsky) and p-CP prediction sets.

\begin{proposition}

Assume the absolute residual calibration scores are i.i.d. and nonnegative, with
$\mu := \mathbb{E}[s_1] \in (0, \infty)$ and $\sigma^2 := \operatorname{Var}(s_1) < \infty$.
Let $F$ denote their cdf, and assume $F$ is continuous at
$q := F^{-1}(1-\alpha)$ with density $f(q) > 0$.
Then, as $n \to \infty$,
\[
\frac{|\mathcal{C}^{\mathbf{ev}}(X_{n+1})|}{|\mathcal{C}^{\mathbf{pv}}(X_{n+1})|}=\frac{S/((n+1)\alpha-1)}{s_{(k)}}\ \longrightarrow\ \frac{\mu}{\alpha q}\ \ge\ 1.
\]
\end{proposition}
\begin{proof}
Let $S=\sum_{i=1}^n s_i$ and $s_{(k)}$ be the $k$th order statistic with $k=\lceil (n+1)(1-\alpha)\rceil$. By the strong law of large numbers, $S/n\to\mu$, and $((n+1)\alpha-1)/n\to \alpha$, hence $S/((n+1)\alpha-1)\to \mu/\alpha$ a.s. By Glivenko--Cantelli and continuity of $F$ at $q:=F^{-1}(1-\alpha)$ with $f(q)>0$, we have $s_{(k)}\to q$ a.s. The continuous mapping theorem yields $\frac{S/((n+1)\alpha-1)}{s_{(k)}}\ \xrightarrow{\text{a.s.}}\ \frac{\mu}{\alpha q}.$ This ratio is greater than one because
$\mathbb{E}[s_1]\ge t\,\mathbb{P}(s_1\ge t)$ by MI, and by continuity at $q$, $\mathbb{P}(s_1\ge q)=\alpha,$ so $  \mu=\mathbb{E}[s_1]\ge q\,\alpha$.
\end{proof}

Our empirical findings are consistent with the theoretical result above. For instance, for the Linear Regression model, the \emph{California housing} dataset (Table~\ref{appendix:table_regression}), the ratio of the average interval lengths between the two sets is $10.717/2.194 \approx4.884$. Figure~\ref{appendix:figure_calib°distribution} shows the empirical distribution of the calibration scores for Linear Regression model, together with the estimated mean and $(1-\alpha)$-quantile. We observe that
\[
\frac{\hat{\mu}}{\alpha \hat{q}} = \frac{0.535}{0.1\times1.094} \approx 4.885,
\]
which matches the theoretical prediction and explains the substantial over-coverage.




\begin{figure}
    \centering
    \includegraphics[width=0.75\linewidth]{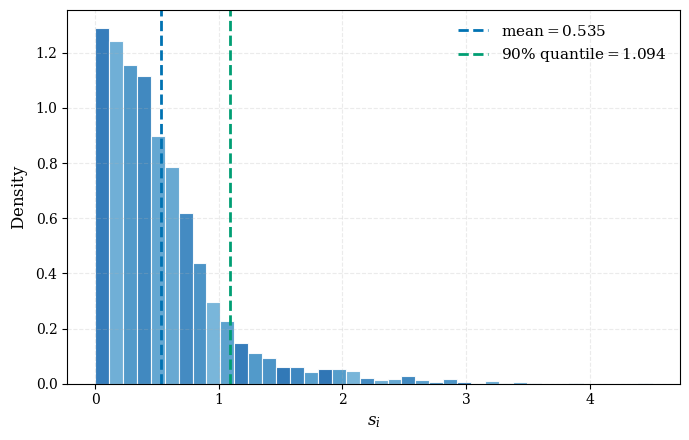}
    \caption{Empirical distribution of calibration scores  (California Housing) for Linear Regression model.}
    \label{appendix:figure_calib°distribution}
\end{figure}

\clearpage
\section{Proofs of results}

We assume in all the following that $\alpha(n+1)\notin \mathbb{N}$ and $\alpha(n+1)>1.$

\label{proffs_appendix}
\subsection{Set-inflation of p-to-e calibrators}

P-values and e-values can be converted into one another by means of \emph{p-to-e calibrators} \citep{vovk2021values,evaluespoly} :

\begin{definition}
A p-to-e calibrator is a decreasing function 
\( f : [0, \infty) \to [0, \infty] \) such that, for any p-variable \(P\) for any hypothesis \(\mathcal H\),
the random variable \(f(P)\) is an e-variable for \(\mathcal H\).
\end{definition}


A useful characterization of the class of p-to-e calibrators is given in the following proposition.
\begin{proposition}
\label{prop:ptoe_carac}
 A decreasing function \( f : [0, \infty) \to [0, \infty] \) with \( f = 0 \) on \( (1, \infty) \)
is a p-to-e calibrator if and only if $\int_0^1 f(p)\,dp \le 1.$
\end{proposition}


However, we show that (almost) no p-to-e calibrator is set-preserving:

\paragraph{Proof of Proposition \ref{prop:only_aon_set_preserving}}

\begin{proof}
    
This comes from the fact that  \( f(p) \le 1/p \) for all \( p \) when \( f \) is a p-to-e calibrator. Indeed, if $\exists p^*\in(0,1)$ such that $f(p^*)>1/p^*$, then \(\int_0^1 f(q)\,dq \ge p^* f(p^*) > 1\), thus violating the condition in Proposition~\ref{prop:ptoe_carac}. Therefore, for any  p-to-e calibrator $f$,
$$\mathcal{C}^{\textbf{pv}}(X_{n+1}) \subseteq \{\, f(P_n) < 1/\alpha \,\}.$$
  Thus,

  $$|\mathcal{C}^{\textbf{pv}}(X_{n+1})|\le |\mathcal{C}^{\textbf{ev}}(X_{n+1})| $$
for any p-to-e calibrator $f$. 
Moreover, it is straightforward to see that $F_{\text{AoN}}$ is set-preserving, by definition of this calibrator.

Now to prove that only the AoN calibrator is set-preserving, assume that \(F\) is a p-to-e calibrator, left-continuous at \(\alpha\), and set-preserving for every nonconformity score used to define \(P_n\), i.e.
\[
\{y:P_n(y)>\alpha\}=\{y:F(P_n(y))<1/\alpha\}
\]
for all \(n\), and all score functions. Since \(F\) is a p-to-e calibrator, \(F(p)\le 1/p\), hence \(F(\alpha)\le 1/\alpha\). Suppose \(F(\alpha)<1/\alpha\). By left-continuity, choose \(\varepsilon>0\) such that \(F(p)<1/\alpha\) for \(p\in(\alpha-\varepsilon,\alpha]\). Pick \(N=n+1\) large with \(q_N=\lfloor \alpha N\rfloor/N\in(\alpha-\varepsilon,\alpha]\). Because set-preservation is assumed for every score, consider the absolute-residual score \(s(x,y)=|y-\mu(x)|\), assuming distinct calibration scores; choosing \(y_0\) so that exactly \(\lfloor\alpha N\rfloor-1\) calibration scores at least as large as \(s(X_{n+1},y_0)\) gives \(P_n(y_0)=q_N\). Then \(q_N\le\alpha\), so \(y_0\notin\{P_n>\alpha\}\), but \(F(P_n(y_0))<1/\alpha\), so \(y_0\in\{F(P_n)<1/\alpha\}\), contradiction. Thus \(F(\alpha)=1/\alpha\). Since \(F\) is decreasing, \(F(p)\ge1/\alpha\) for \(p\le\alpha\), and \(\int_0^1F\le1\) forces \(F(p)=1/\alpha\) on \((0,\alpha]\) and \(F(p)=0\) on \((\alpha,1]\). Therefore \(F=F_{\mathrm{AoN}}\).

\end{proof}

\subsection{Proof of Proposition \ref{prop:fnalpha}}

\begin{proof}
Under the null and assuming no ties among the nonconformity scores, the p-variable is uniformly distributed on $\{\frac{1}{n+1},\cdots, \frac{n}{n+1}, 1\}$. Thus, for $f\in\mathcal{E}$, condition \eqref{E(F)=1} is equivalent to 
\begin{equation}
\label{starprop}
    \frac{1}{n+1}\sum_{k=1}^{n+1}f\left(\frac{k}{n+1}\right) = \alpha f(\alpha).  \tag{$\star$}
\end{equation}

 We will need the following lemma to complete the proof.
\begin{lemma}\label{lem:balance}
Let $V\in\mathbb{R}^d$ have at least one positive and at least one negative coordinate.
Then there exists $x\in\mathbb{R}^d_{>0}$ such that $V\cdot x=0$.
\end{lemma}

\begin{proof}
Take any $x^{(0)}\in\mathbb{R}^d_{>0}$  and set $s:=V\cdot x^{(0)}$.
If $s=0$, the claim holds. If \(s > 0\), choose an index \(i\) such that \(V_i < 0\) and define \(t := s / (-V_i) > 0\). Then \(x := x^{(0)} + t e_i \in \mathbb{R}^d_{>0}\), and $V \cdot x = s + t V_i = s - s = 0$. If \(s < 0\), the same argument applies with an index \(j\) satisfying \(V_j > 0\).
\end{proof}

Let \(N:=n+1\) and \(j:=\lfloor \alpha N\rfloor\), so that
\[
\frac{j}{N}<\alpha<\frac{j+1}{N}.
\]
Define the augmented grid
\[
y_1=\frac1N,\ldots,y_j=\frac jN,\quad
y_{j+1}=\alpha,\quad
y_{j+2}=\frac{j+1}{N},\ldots,y_{N+1}=1,
\]
and set \(u_k=f(y_k)\). Let
\[
\delta_m:=u_m-u_{m+1},\qquad m=1,\ldots,N,
\qquad
c:=u_{N+1}.
\]
Then
\[
u_k=c+\sum_{m=k}^{N}\delta_m,\qquad k=1,\ldots,N.
\]
Since the point \(y_{j+1}=\alpha\) is not part of the conformal grid,
condition \((\star)\) is equivalent to
\[
\frac1N
\left(
\sum_{k=1}^{j}u_k+\sum_{k=j+2}^{N+1}u_k
\right)
=
\alpha u_{j+1}.
\]
Substituting the representation of the \(u_k\)'s gives after simplification
\[
N(1-\alpha)c
+
\sum_{m=1}^{j}m\delta_m
+
(j-\alpha N)\delta_{j+1}
+
\sum_{m=j+2}^{N}(m-1-\alpha N)\delta_m
=
0.
\]
Define
\[
\beta_0:=N(1-\alpha),\qquad
\beta_m:=m,\quad m=1,\ldots,j,
\]
\[
\beta_{j+1}:=j-\alpha N<0,
\qquad
\beta_m:=m-1-\alpha N,\quad m=j+2,\ldots,N.
\]
Then \(\beta_0>0\), \(\beta_{j+1}<0\), and
\(\beta_m>0\) for all \(m\neq j+1\). By Lemma~\ref{lem:balance},
there exists
\[
(c,\delta_1,\ldots,\delta_N)\in\mathbb R_{>0}^{N+1}
\]
such that
\[
\beta_0 c+\sum_{m=1}^{N}\beta_m\delta_m=0.
\]
Thus \(u_1>\cdots>u_{N+1}>0\). Finally, choose any
\(u_0>u_1\) at \(y_0=0\), and define \(f\) by linear interpolation
through
\[
(y_0,u_0),(y_1,u_1),\ldots,(y_{N+1},u_{N+1}).
\]
Then \(f\) is positive and strictly decreasing on \([0,1]\), and by
construction it satisfies \((\star)\).

\end{proof}


\begin{remark} When $\alpha=\tfrac{j}{n+1}$ with $1\le j\le n$, no solution exists. Indeed, writing $u_k:=f\!\Big(\frac{k}{n+1}\Big),$ and $
u_1>\cdots>u_{n+1}>0,$
equation \eqref{starprop} becomes 
\begin{equation}\label{eq:on-grid}
\sum_{k=1}^{n+1} u_k \;=\; j\,u_j
\quad\Longleftrightarrow\quad
\sum_{k\neq j} u_k \;=\; (j-1)\,u_j .
\end{equation}
But strict monotonicity implies $\sum_{k=1}^{j-1} u_k \;>\; (j-1)\,u_j$ and $
\sum_{k=j+1}^{n+1} u_k \;>\; 0, $ thus leading to a contradiction.
\end{remark}
\paragraph{Proof of Proposition \ref{prop:convergence_AON} }

Let \(f_n\in\mathcal E\) be positive and strictly decreasing, satisfying
\[
\frac{1}{n+1}\sum_{k=1}^{n+1}
f_n\!\left(\frac{k}{n+1}\right)
=
\alpha f_n(\alpha).
\tag{\(\star\)}
\]
Then
\[
\frac{f_n(p)}{f_n(\alpha)}
\longrightarrow
\mathbf 1_{\{p\le \alpha\}}
\]
for every fixed \(p\in(0,1]\).

\begin{proof}
By definition we have
\[
0\le \alpha(n+1)-\lfloor \alpha(n+1)\rfloor<1.
\]

Multiplying \((\star)\) by \(n+1\), we get
\[
\sum_{k=1}^{n+1}
f_n\!\left(\frac{k}{n+1}\right)
=
\alpha(n+1)f_n(\alpha).
\]
Subtracting \(\lfloor \alpha(n+1)\rfloor f_n(\alpha)\) from both sides gives
\[
\sum_{k=1}^{\lfloor \alpha(n+1)\rfloor}
\left[
f_n\!\left(\frac{k}{n+1}\right)-f_n(\alpha)
\right]
+
\sum_{k=\lfloor \alpha(n+1)\rfloor+1}^{n+1}
f_n\!\left(\frac{k}{n+1}\right)
=
\bigl(\alpha(n+1)-\lfloor \alpha(n+1)\rfloor\bigr)f_n(\alpha).
\tag{1}
\]
Both terms on the left-hand side are nonnegative. Indeed, if
\(k\le \lfloor \alpha(n+1)\rfloor\), then \(k/(n+1)\le \alpha\), so by monotonicity,
$f_n\!\left(\frac{k}{n+1}\right)\ge f_n(\alpha).$

\begin{itemize}
    \item Now fix \(p<\alpha\). We prove that
    \[
    \frac{f_n(p)}{f_n(\alpha)}\to 1.
    \]
    Since \(f_n\) is decreasing and \(p<\alpha\), we already have
    $\frac{f_n(p)}{f_n(\alpha)}\ge 1.$
    For every integer \(k\le \lfloor p(n+1)\rfloor\), we have
    \[
    \frac{k}{n+1}\le p<\alpha
    \implies
    f_n\!\left(\frac{k}{n+1}\right)\ge f_n(p).
    \]
    Therefore,
    \[
    f_n\!\left(\frac{k}{n+1}\right)-f_n(\alpha)
    \ge
    f_n(p)-f_n(\alpha).
    \]
    Summing over \(k=1,\dots,\lfloor p(n+1)\rfloor\), we get
    \[
    \sum_{k=1}^{\lfloor p(n+1)\rfloor}
    \left[
    f_n\!\left(\frac{k}{n+1}\right)-f_n(\alpha)
    \right]
    \ge
    \lfloor p(n+1)\rfloor
    \left[
    f_n(p)-f_n(\alpha)
    \right].
    \]
    The left-hand side is bounded above by the first sum in (1), and the
    first sum in (1) is bounded above by the right-hand side of (1).
    Hence
    \[
    \lfloor p(n+1)\rfloor
    \left[
    f_n(p)-f_n(\alpha)
    \right]
    \le
    \bigl(\alpha(n+1)-\lfloor \alpha(n+1)\rfloor\bigr)f_n(\alpha).
    \]
    Since
    \[
    0\le \alpha(n+1)-\lfloor \alpha(n+1)\rfloor<1,
    \]
    we obtain
    \[
    0
    \le
    \frac{f_n(p)}{f_n(\alpha)}-1
    \le
    \frac{1}{\lfloor p(n+1)\rfloor}.
    \]
    Now considering an  increasing sequence $n_k$, with $\alpha(n_k+1)\notin \mathbb{N},$ with $ n_k\to\infty, k\to \infty$, the right-hand side goes to \(0\). Therefore,
    \[
    \frac{f_n(p)}{f_n(\alpha)}\to 1.
    \]

    \item Fix \(p>\alpha\). We now prove that
    $\frac{f_n(p)}{f_n(\alpha)}\to 0$.
    For a sufficiently large $n$, consider \(k\) s.t. $$\alpha<\frac{k}{n+1}\le p\implies
    f_n\!\left(\frac{k}{n+1}\right)\ge f_n(p).$$
    Therefore,
    \[
    \sum_{k=\lfloor \alpha(n+1)\rfloor+1}^{\lfloor p(n+1)\rfloor}
    f_n\!\left(\frac{k}{n+1}\right)
    \ge
    \left(
    \lfloor p(n+1)\rfloor-\lfloor \alpha(n+1)\rfloor
    \right)f_n(p).
    \]
    The left-hand side is bounded above by the second sum in (1), and the
    second sum in (1) is bounded above by the right-hand side of (1).
    Hence
    \[
    \left(
    \lfloor p(n+1)\rfloor-\lfloor \alpha(n+1)\rfloor
    \right)f_n(p)
    \le
    \bigl(\alpha(n+1)-\lfloor \alpha(n+1)\rfloor\bigr)f_n(\alpha).
    \]
    Since
    \[
    0\le \alpha(n+1)-\lfloor \alpha(n+1)\rfloor<1,
    \]
    we obtain
    \[
    0
    \le
    \frac{f_n(p)}{f_n(\alpha)}
    \le
    \frac{1}{
    \lfloor p(n+1)\rfloor-\lfloor \alpha(n+1)\rfloor
    }.
    \]
    Because \(p>\alpha\), we have
    $\lfloor p(n+1)\rfloor-\lfloor \alpha(n+1)\rfloor\to\infty.$
    Hence
    \[
    \frac{f_n(p)}{f_n(\alpha)}\to 0.
    \]

    \item Finally, if \(p=\alpha\), then trivially
    $\frac{f_n(\alpha)}{f_n(\alpha)}=1$.
\end{itemize}

Combining the three cases, we conclude that, for every fixed \(p\in(0,1]\),
\[
\frac{f_n(p)}{f_n(\alpha)}
\longrightarrow
\mathbf 1_{\{p\le \alpha\}},
\]
and thus $F_{n,\alpha}\to F_{\mathrm{AoN}}.$
\end{proof}

\subsection{Proof of Theorem \ref{main_theorem}}

\begin{proof}
Define
\[
\mathcal{L}(C,p)
:= \frac{1}{\alpha} \cdot 
\frac{1 + \exp\!\big(C(\alpha - s)\big)}
     {1 + \exp\!\big(C(p - s)\big)},
\qquad C > 0, \; p > 0.
\]
For each fixed \(C > 0\), the mapping \(p \mapsto \mathcal{L}(C,p)\) is strictly decreasing and positive. 
Hence, we seek \(C > 0\) such that
\begin{equation}\label{star}
\mathbb{E}_{\mathbb{H}_0}\!\left(E_{n,\alpha}\right) = 1
\quad \Longleftrightarrow \quad
\Sigma(C) 
:= \frac{1}{n+1} \sum_{k=1}^{n+1} 
\mathcal{L}\!\left(C, \tfrac{k}{n+1}\right)
= 1.
\end{equation}
The function \(\Sigma(C)\) is continuous on \((0,\infty)\). 
We examine its limiting behavior:

\smallskip
\noindent
\(\bullet\) As \(C \to 0^+\), we have 
\(\exp\!\big(C(\alpha - s)\big) \to 1\), and thus each term of the sum tends to \(1/\alpha\); consequently, 
\[
\lim_{C \to 0^+} \Sigma(C) = \tfrac{1}{\alpha} > 1.
\]

\smallskip
\noindent
\(\bullet\) As \(C \to \infty\), \(\exp\!\big(C(\alpha - s)\big) \to 0\), and
\[
\frac{1}{1 + \exp\!\big(C(\tfrac{k}{n+1} - s)\big)} 
\longrightarrow 
\begin{cases}
1, & \tfrac{k}{n+1} < s,\\[4pt]
0, & \tfrac{k}{n+1} > s.
\end{cases}
\]
Hence,
\[
\lim_{C \to \infty} \Sigma(C)
= \frac{1}{\alpha} \cdot 
\frac{\#\{\,k : \tfrac{k}{n+1} < s\,\}}{n+1}
= \frac{\lfloor \alpha (n+1) \rfloor}{\alpha (n+1)} < 1.
\]

\smallskip
\noindent
By continuity of \(\Sigma\), the Intermediate Value Theorem guarantees the existence of a 
\(C_{n,\alpha} > 0\) such that \(\Sigma(C_{n,\alpha}) = 1\).

\paragraph{Properties of our P2E calibrator.}
For fixed $n,\alpha$, our P2E calibrator
\[
F_{n,\alpha}(p)
:= \frac{1}{\alpha}\cdot
\frac{1+\exp\!\big(C_{n,\alpha}(\alpha-s)\big)}
     {1+\exp\!\big(C_{n,\alpha}(p-s)\big)}
\]
is smooth thanks to its sigmoid-like form. Noting that $C_{n,\alpha}>0$, it is strictly decreasing as a function of $p$, since
\[
\frac{\partial}{\partial p}F_{n,\alpha}(p)
=
-\frac{C_{n,\alpha}}{\alpha}\,
\frac{\big(1+\exp(C_{n,\alpha}(\alpha-s))\big)
      \exp(C_{n,\alpha}(p-s))}
     {\big(1+\exp(C_{n,\alpha}(p-s))\big)^2}
<0 .
\]
Moreover, it satisfies $F_{n,\alpha}(\alpha)=\frac{1}{\alpha}.$
Its inverse is obtained explicitly as follows. For $e=F_{n,\alpha}(p)$,
\[
p
=
F_{n,\alpha}^{-1}(e)
=
s+\frac{1}{C_{n,\alpha}}
\log\!\left(
\frac{1+\exp(C_{n,\alpha}(\alpha-s))}
     {\alpha e}
-1
\right),
\]
for $e$ in the range of $F_{n,\alpha}$.

\paragraph{Dominance Over AoN.}

For
\[
F_{n,\alpha}(p) = \frac{1}{\alpha} \frac{1 + \exp(C_{n,\alpha}(\alpha - s))}{1 + \exp(C_{n,\alpha}(p - s))}
\]
with $s > \alpha$:
\begin{itemize}
    \item If $p \le \alpha$, then $p - s \le \alpha - s$, so
    \[
    \exp(C(p - s)) \le \exp(C(\alpha - s)),
    \]
    hence the denominator is no larger than the numerator and
    \[
    F_{n,\alpha}(p) \ge 1/\alpha = F_{\text{AoN}}(p).
    \]
    \item If $p > \alpha$, then $F_{\text{AoN}}(p) = 0$, while $F_{n,\alpha}(p) > 0$.
\end{itemize}
Therefore $F_{n,\alpha} \ge F_{\text{AoN}}$ pointwise.

\end{proof}

\subsection{Proof sketch of Proposition \ref{prop:ECCP}}

It follows directly from the Markov-type bounds of Subsection \ref{subsec:evals_prop}. Applying MI~\eqref{Prop:markov_coverage} or UR-MI~\eqref{prop:URMI} 
then yields the stated finite-sample coverage guarantees. 

For the exchangeable constructions 
of ECCP-Exch and UR-ECCP-Exch,
the corresponding results follow by invoking the exchangeable 
Markov inequality (EMI) in~\eqref{prop:UR-EMI}.


\section{ECCP improves the stability of CCP}
\label{app:ECCP_RF}
Empirical studies confirm that CCP can either under-cover or remain overly conservative depending on the stability of the underlying model \citep{linusson2017calibration}. To illustrate this, Figure~\ref{fig:boston_rf_ccp_eccp} shows how coverage and prediction set size evolve with the number of trees in a Random Forest (RF) when using CCP and ECCP. With few trees, when the RF is unstable, CCP becomes markedly conservative, in line with \citet{linusson2017calibration}.  As the number of trees increases and the RF stabilizes, CCP's coverage moves closer to the nominal level. In contrast, ECCP consistently remains near the target coverage regardless of the stability of the RF model.

\begin{figure}
\centering
\begin{minipage}{0.49\columnwidth}
  \centering
  \includegraphics[width=\linewidth]{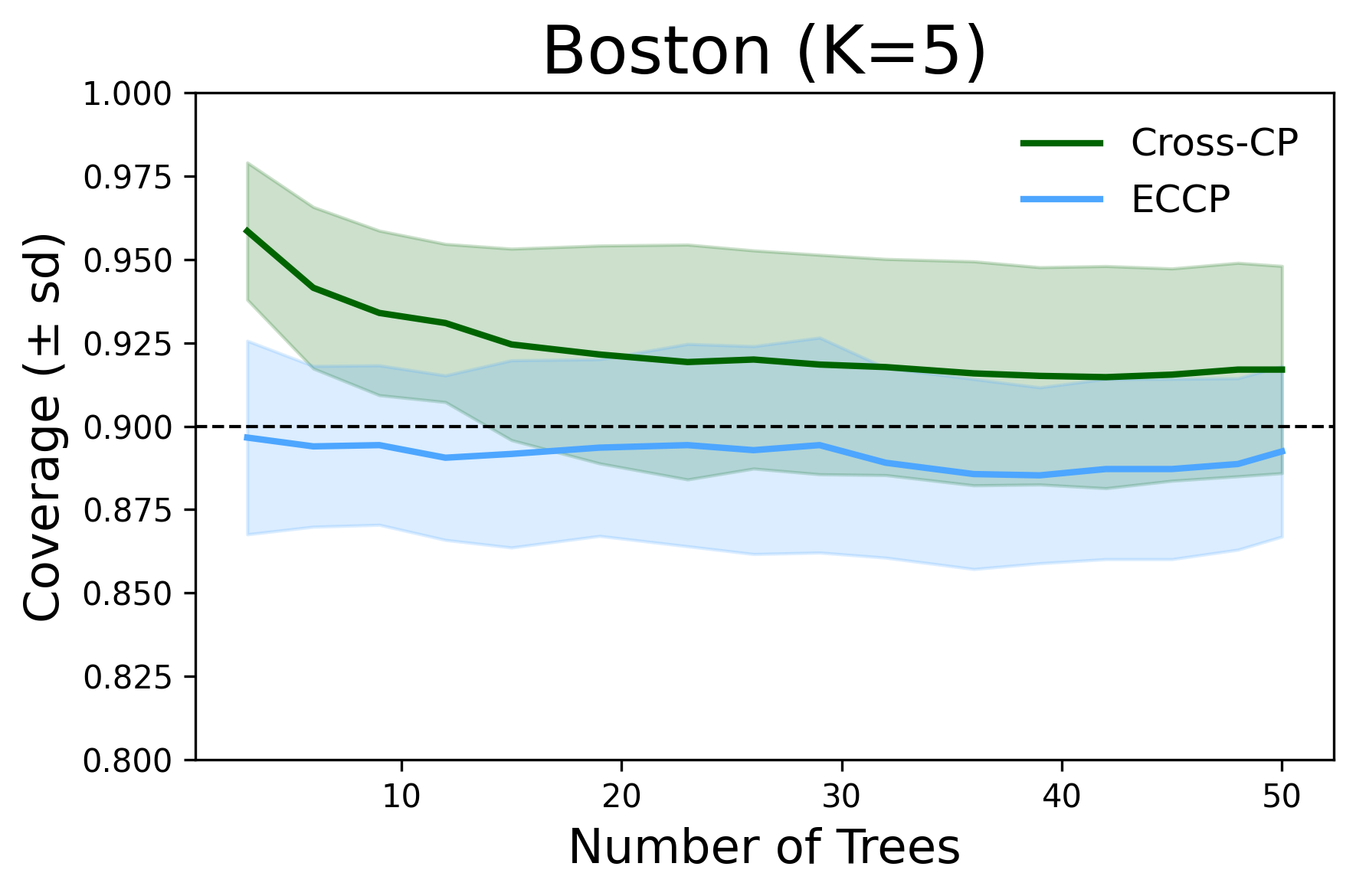}

\end{minipage}\hfill
\begin{minipage}{0.49\columnwidth}
  \centering
  \includegraphics[width=\linewidth]{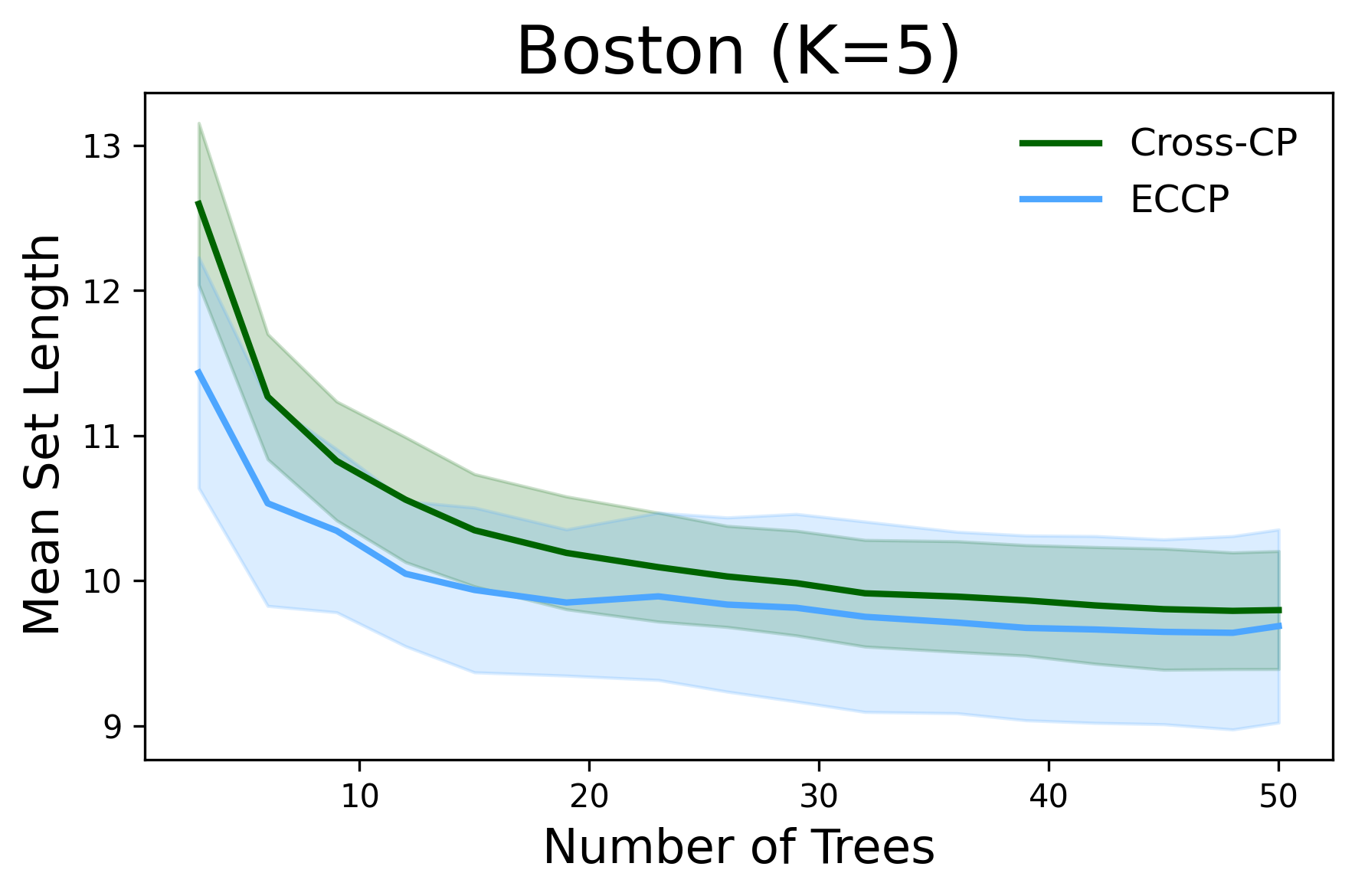}

\end{minipage}

\vspace{0.2em}

\begin{minipage}{0.49\columnwidth}
  \centering
  \includegraphics[width=\linewidth]{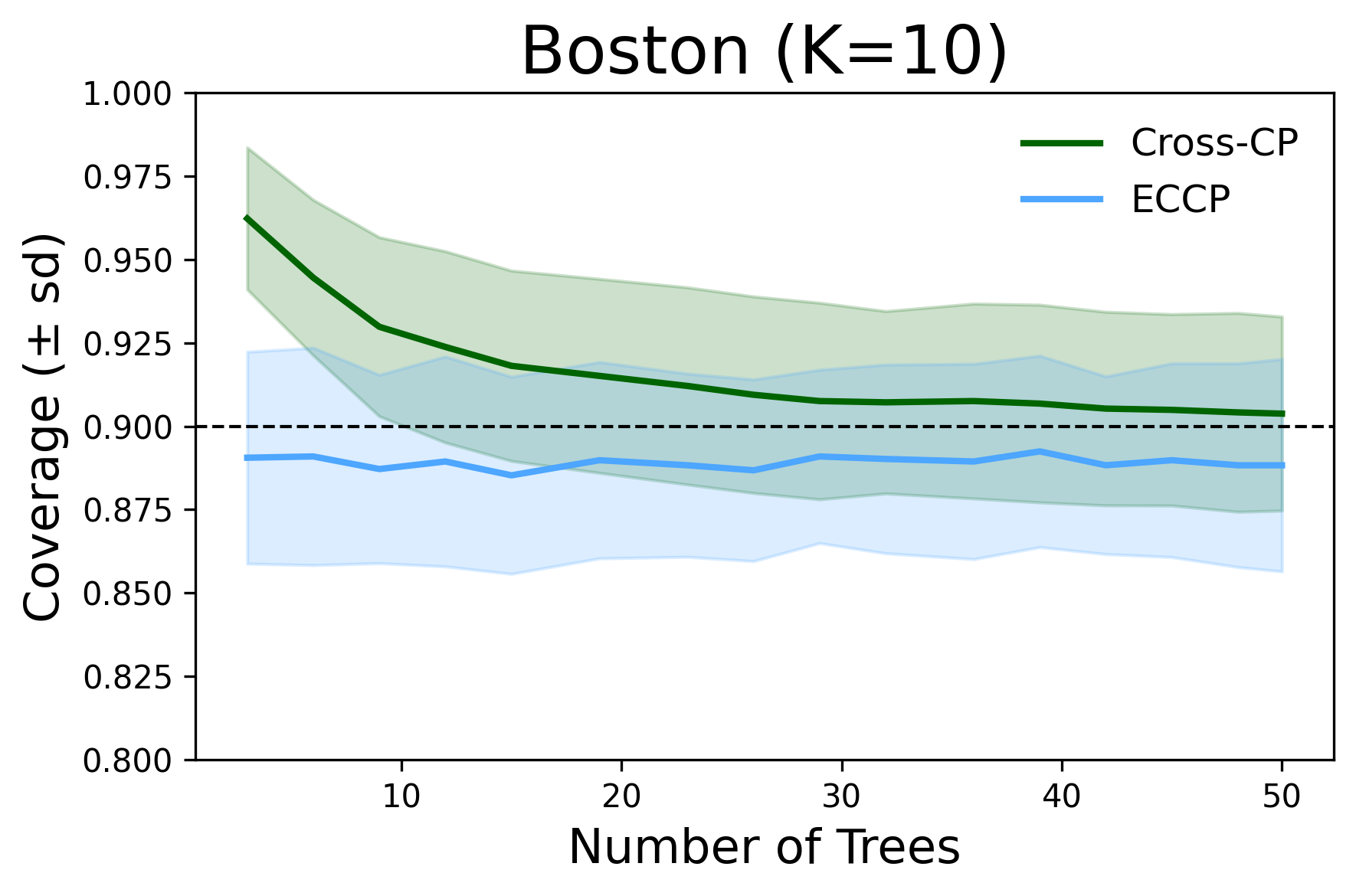}

\end{minipage}\hfill
\begin{minipage}{0.49\columnwidth}
  \centering
  \includegraphics[width=\linewidth]{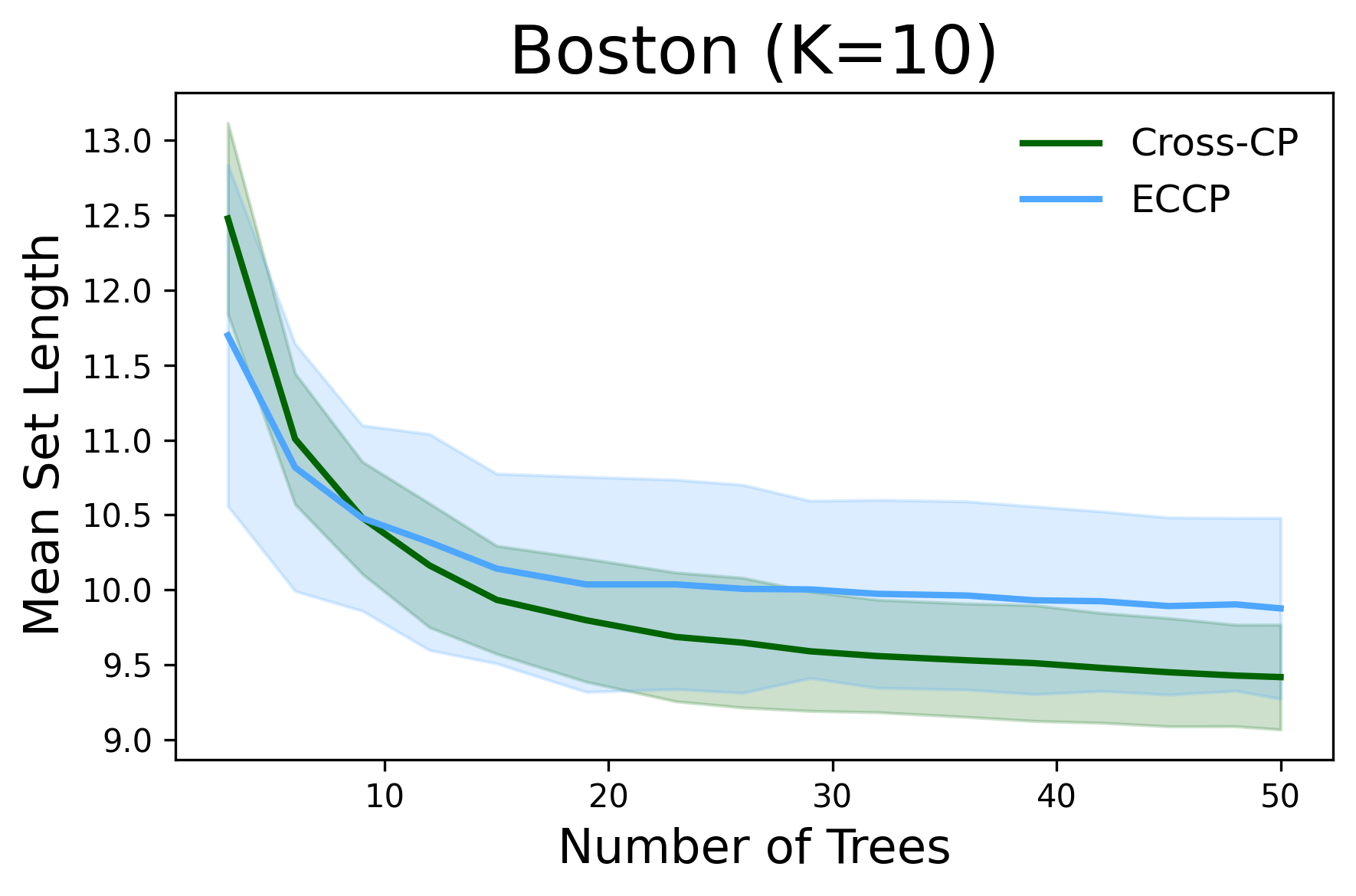}

\end{minipage}

\caption{Evolution of coverage and mean prediction set size on the Boston dataset as a function of the number of trees in the RF model, for both CCP (green) and ECCP (blue), averaged over 25 seeds. The surrounding shaded area indicates the range of mean $\pm$ standard deviation. Miscoverage level is set to $\alpha = 0.1$.}

\label{fig:boston_rf_ccp_eccp}
\end{figure}

\section{ECCP Additional experiments}
\label{app:additional_experiments}
\subsection{ECCP methods}
We provide additional experimental results for the CCP methods on Boston, Abalone, and Parkinson
\citep{tsanas2009parkinsons} datasets,  
for three base models (OLS, RF, and Lasso) as in \citep{gasparinimproving}. 

\begin{table}[H]
\centering
\caption{Comparison of prediction set size and coverage, Boston dataset, $K=15$, across 100 random seeds.}
\begin{tabular}{llccccc}
\toprule
Base & Metric & CCP & e-mod-cross & u-mod-cross & eu-mod-cross & ECCP$(2\alpha)$ \\
\midrule
OLS & Size & 13.970 $\pm$ 0.456 & 14.100 $\pm$ 1.720 & 13.288 $\pm$ 0.439 & 12.569 $\pm$ 1.955 & 10.665 $\pm$ 0.473 \\
 & Cov. & 0.897 $\pm$ 0.032 & 0.895 $\pm$ 0.055 & 0.880 $\pm$ 0.034 & 0.854 $\pm$ 0.067 & 0.793 $\pm$ 0.042 \\
\midrule
RF & Size & 9.140 $\pm$ 0.357 & 9.208 $\pm$ 1.450 & 8.615 $\pm$ 0.323 & 7.994 $\pm$ 1.618 & 6.549 $\pm$ 0.317 \\
 & Cov. & 0.902 $\pm$ 0.035 & 0.894 $\pm$ 0.049 & 0.882 $\pm$ 0.034 & 0.849 $\pm$ 0.070 & 0.784 $\pm$ 0.046 \\
\midrule
Lasso & Size & 13.926 $\pm$ 0.463 & 14.109 $\pm$ 1.736 & 13.268 $\pm$ 0.447 & 12.548 $\pm$ 2.006 & 10.632 $\pm$ 0.480 \\
 & Cov. & 0.898 $\pm$ 0.032 & 0.895 $\pm$ 0.053 & 0.881 $\pm$ 0.032 & 0.854 $\pm$ 0.066 & 0.792 $\pm$ 0.043 \\
\bottomrule
\end{tabular}
\end{table}

\begin{table}[H]
\centering

\begin{tabular}{llccccc}
\toprule
Base & Metric & ECCP & ECCP($F_{\text{AoN}}$) & ECCP($F_{1}$) & ECCP($F_2$) & ECCP($F_3$) \\
\midrule
OLS & Size & 15.458 $\pm$ 0.850 & 18.727 $\pm$ 1.300 & 65.724 $\pm$ 3.768 & 71.783 $\pm$ 3.746 & 63.206 $\pm$ 3.979 \\
 & Cov. & 0.905 $\pm$ 0.033 & 0.926 $\pm$ 0.029 & 0.926 $\pm$ 0.029 & 0.907 $\pm$ 0.031 & 0.908 $\pm$ 0.031 \\
\midrule
RF & Size & 10.126 $\pm$ 0.616 & 12.662 $\pm$ 1.153 & 63.487 $\pm$ 4.036 & 70.362 $\pm$ 3.932 & 61.264 $\pm$ 4.204 \\
 & Cov. & 0.899 $\pm$ 0.032 & 0.922 $\pm$ 0.027 & 0.921 $\pm$ 0.024 & 0.907 $\pm$ 0.028 & 0.903 $\pm$ 0.026 \\
\midrule
Lasso & Size & 15.479 $\pm$ 0.839 & 18.786 $\pm$ 1.309 & 65.746 $\pm$ 3.768 & 71.786 $\pm$ 3.745 & 63.204 $\pm$ 3.984 \\
 & Cov. & 0.903 $\pm$ 0.031 & 0.926 $\pm$ 0.028 & 0.926 $\pm$ 0.029 & 0.907 $\pm$ 0.032 & 0.908 $\pm$ 0.030 \\
\bottomrule
\end{tabular}
\end{table}


\begin{table}[H]
\centering
\caption{Comparison of prediction set size and coverage, Abalone dataset, $K=15$, across 100 random seeds.}
\begin{tabular}{llccccc}
\toprule
Base & Metric & CCP & e-mod-cross & u-mod-cross & eu-mod-cross & ECCP$(2\alpha)$ \\
\midrule
OLS & Size & 6.603 $\pm$ 0.034 & 6.365 $\pm$ 0.300 & 5.477 $\pm$ 0.050 & 5.444 $\pm$ 0.324 & 4.600 $\pm$ 0.036 \\
 & Cov. & 0.898 $\pm$ 0.022 & 0.888 $\pm$ 0.027 & 0.844 $\pm$ 0.027 & 0.840 $\pm$ 0.034 & 0.787 $\pm$ 0.033 \\
\midrule
RF & Size & 6.696 $\pm$ 0.057 & 6.358 $\pm$ 0.282 & 5.547 $\pm$ 0.064 & 5.441 $\pm$ 0.399 & 4.557 $\pm$ 0.063 \\
 & Cov. & 0.901 $\pm$ 0.024 & 0.884 $\pm$ 0.029 & 0.848 $\pm$ 0.029 & 0.838 $\pm$ 0.038 & 0.783 $\pm$ 0.036 \\
\midrule
Lasso & Size & 6.513 $\pm$ 0.035 & 6.302 $\pm$ 0.255 & 5.497 $\pm$ 0.046 & 5.472 $\pm$ 0.309 & 4.659 $\pm$ 0.036 \\
 & Cov. & 0.900 $\pm$ 0.022 & 0.891 $\pm$ 0.025 & 0.846 $\pm$ 0.029 & 0.844 $\pm$ 0.034 & 0.790 $\pm$ 0.032 \\
\bottomrule
\end{tabular}
\end{table}

\begin{table}[H]
\centering

\begin{tabular}{llccccc}
\toprule
Base & Metric & ECCP & ECCP($F_{\text{AoN}}$) & ECCP($F_1$) & ECCP($F_2$) & ECCP($F_3$) \\
\midrule
OLS & Size & 6.716 $\pm$ 0.076 & 6.755 $\pm$ 0.080 & 10.350 $\pm$ 0.327 & 30.647 $\pm$ 1.956 & 32.181 $\pm$ 1.975 \\
 & Cov. & 0.898 $\pm$ 0.021 & 0.899 $\pm$ 0.021 & 0.901 $\pm$ 0.022 & 0.891 $\pm$ 0.024 & 0.891 $\pm$ 0.023 \\
\midrule
RF & Size & 6.747 $\pm$ 0.087 & 6.788 $\pm$ 0.086 & 10.294 $\pm$ 0.329 & 30.578 $\pm$ 1.958 & 32.133 $\pm$ 1.976 \\
 & Cov. & 0.897 $\pm$ 0.024 & 0.899 $\pm$ 0.024 & 0.904 $\pm$ 0.025 & 0.893 $\pm$ 0.026 & 0.893 $\pm$ 0.024 \\
\midrule
Lasso & Size & 6.665 $\pm$ 0.082 & 6.704 $\pm$ 0.083 & 10.670 $\pm$ 0.340 & 30.772 $\pm$ 1.946 & 32.233 $\pm$ 1.971 \\
 & Cov. & 0.899 $\pm$ 0.022 & 0.900 $\pm$ 0.022 & 0.902 $\pm$ 0.024 & 0.892 $\pm$ 0.024 & 0.891 $\pm$ 0.024 \\
\bottomrule
\end{tabular}
\end{table}

\begin{table}[H]
\centering
\caption{Comparison of prediction set size and coverage, Parkinsons dataset, $K=20$, across 100 random seeds.}
\begin{tabular}{llccccc}
\toprule
Base & Metric & CCP & e-mod-cross & u-mod-cross & eu-mod-cross & ECCP$(2\alpha)$ \\
\midrule
OLS & Size & 29.342 $\pm$ 0.429 & 28.462 $\pm$ 1.244 & 26.239 $\pm$ 0.256 & 26.087 $\pm$ 1.418 & 23.462 $\pm$ 0.236 \\
 & Cov. & 0.897 $\pm$ 0.008 & 0.886 $\pm$ 0.019 & 0.850 $\pm$ 0.009 & 0.846 $\pm$ 0.027 & 0.792 $\pm$ 0.010 \\
\midrule
RF & Size & 5.537 $\pm$ 0.426 & 5.014 $\pm$ 0.600 & 4.226 $\pm$ 0.348 & 3.996 $\pm$ 0.663 & 2.978 $\pm$ 0.282 \\
 & Cov. & 0.898 $\pm$ 0.008 & 0.877 $\pm$ 0.018 & 0.851 $\pm$ 0.008 & 0.830 $\pm$ 0.030 & 0.772 $\pm$ 0.010 \\
\midrule
Lasso & Size & 29.330 $\pm$ 0.417 & 28.492 $\pm$ 1.213 & 26.243 $\pm$ 0.254 & 26.088 $\pm$ 1.415 & 23.468 $\pm$ 0.237 \\
 & Cov. & 0.897 $\pm$ 0.008 & 0.886 $\pm$ 0.018 & 0.850 $\pm$ 0.009 & 0.846 $\pm$ 0.027 & 0.793 $\pm$ 0.010 \\
\bottomrule
\end{tabular}
\end{table}

\begin{table}[H]
\centering

\begin{tabular}{llccccc}
\toprule
Base & Metric & ECCP & ECCP($F_{\text{AoN}}$) & ECCP($F_1$) & ECCP($F_2$) & ECCP($F_3$) \\
\midrule
OLS & Size & 30.100 $\pm$ 0.528 & 30.154 $\pm$ 0.532 & 39.431 $\pm$ 0.335 & 70.983 $\pm$ 0.667 & 70.004 $\pm$ 0.688 \\
 & Cov. & 0.897 $\pm$ 0.008 & 0.898 $\pm$ 0.008 & 0.910 $\pm$ 0.005 & 0.900 $\pm$ 0.005 & 0.899 $\pm$ 0.005 \\
\midrule
RF & Size & 5.627 $\pm$ 0.422 & 5.648 $\pm$ 0.423 & 13.145 $\pm$ 0.748 & 59.823 $\pm$ 0.853 & 59.351 $\pm$ 0.880 \\
 & Cov. & 0.890 $\pm$ 0.009 & 0.891 $\pm$ 0.009 & 0.898 $\pm$ 0.007 & 0.885 $\pm$ 0.007 & 0.884 $\pm$ 0.006 \\
\midrule
Lasso & Size & 30.093 $\pm$ 0.524 & 30.147 $\pm$ 0.527 & 39.438 $\pm$ 0.338 & 70.983 $\pm$ 0.667 & 70.006 $\pm$ 0.688 \\
 & Cov. & 0.898 $\pm$ 0.008 & 0.898 $\pm$ 0.008 & 0.910 $\pm$ 0.005 & 0.900 $\pm$ 0.005 & 0.899 $\pm$ 0.005 \\
\bottomrule
\end{tabular}
\end{table}

Overall, among the $1-2\alpha$ methods, our $\mathrm{ECCP}$ is the most efficient, consistently outperforming the competing baselines. For the $1-\alpha$ methods, we observe the same pattern as before: $\mathrm{ECCP}$ comes with a large-sample performance guarantee of $F_{\mathrm{AoN}}$ and remains substantially more efficient than alternatives based on set-inflating calibrators, such as $F_1(p)=-\log p$, $F_2(p)=p^{-1/2}-1$, and $F_3(p)=2(1-p)$.

\section{Details on the P2E calibrator}
\subsection{The $\alpha(n+1)\notin \mathbb{N}$ condition}

Our framework is based on the assumption \(\alpha(n+1)\notin\mathbb N\), which makes our main theorem valid. This assumption is only a mild technical condition. It is used to ensure that the interval
$\left(\alpha,\frac{\lceil \alpha(n+1)\rceil}{n+1}\right)$
is nonempty, so that the parameter \(s\) in the P2E calibrator can be chosen strictly larger than \(\alpha\).  For example, for $\alpha=0.1$, it occurs for \(n=9,19,29,\ldots\). Additionally, if $\alpha$ is irrational, then this case never happens, although in practice we usually consider $\alpha$ to be $0.1, 0.05, 0.025$, etc.  Since this case occurs for a condition only dependent on the calibration size and the miscoverage level, it can be avoided simply, by removing or adding one calibration point.  Indeed, if \(\alpha(n+1)=k\in\mathbb N\), then $\alpha n=k-\alpha\notin\mathbb{N}$ and $\alpha(n+2) = k+\alpha\notin\mathbb{N}$ for any $\alpha\in (0,1).$ 
\subsection{Sensitivity to the choice of \(s\).}
We study the effect of the tuning parameter
$s\in\left(\alpha,\frac{\lceil \alpha(n+1)\rceil}{n+1}\right)$
on the performance of our methods: 
\[
s_w
=
w\alpha
+
(1-w)
\frac{\lceil \alpha(n+1)\rceil}{n+1}\qquad \text{for \(w\in(0,1)\)}.
\]
and report in the tables below the average prediction set length and empirical coverage vary for different values of $w$ for the three base predictors we consider. 

We observe that the performance varies only slightly across the different values of \(w\), and hence across the corresponding choices of
$s_w
=
w\alpha
+
(1-w)\frac{\lceil \alpha(n+1)\rceil}{n+1}
\in
\left(\alpha,\frac{\lceil \alpha(n+1)\rceil}{n+1}\right).$
This suggests that the tuning parameter \(s\) has little practical impact on the transformed e-variable produced by our P2E calibrator, and consequently on both the efficiency and the coverage. This stability can also be understood from the sigmoid form of the calibrator in \eqref{final_evalue}: geometrically, small changes in the location parameter \(s\) induce only mild changes in the shape of the map $p\to F_{n,\alpha}(p)$.

\begin{table}[H]
\centering
\caption{Mean length and coverage for OLS across 30 seeds for different weights, Boston dataset, $\alpha=0.1, K=5$.}
\label{tab:ols_weights}
\resizebox{\textwidth}{!}{
\begin{tabular}{lcccccc}
\toprule
$[w,1-w]$ & \multicolumn{2}{c}{ECCP} & \multicolumn{2}{c}{ECCP-exch} & \multicolumn{2}{c}{UR-ECCP-exch} \\
\cmidrule(lr){2-3} \cmidrule(lr){4-5} \cmidrule(lr){6-7}
 & Len. & Cov. & Len. & Cov. & Len. & Cov. \\
\midrule
$[0.05, 0.95]$ & 14.351 $\pm$ 0.564 & 0.895 $\pm$ 0.026 & 14.898 $\pm$ 2.172 & 0.900 $\pm$ 0.050 & 14.811 $\pm$ 2.123 & 0.898 $\pm$ 0.051 \\
$[0.1, 0.9]$ & 14.351 $\pm$ 0.564 & 0.895 $\pm$ 0.026 & 14.898 $\pm$ 2.172 & 0.900 $\pm$ 0.050 & 14.811 $\pm$ 2.123 & 0.898 $\pm$ 0.051 \\
$[0.2, 0.8]$ & 14.351 $\pm$ 0.565 & 0.895 $\pm$ 0.026 & 14.898 $\pm$ 2.172 & 0.900 $\pm$ 0.050 & 14.812 $\pm$ 2.124 & 0.898 $\pm$ 0.051 \\
$[0.3, 0.7]$ & 14.359 $\pm$ 0.570 & 0.895 $\pm$ 0.026 & 14.898 $\pm$ 2.172 & 0.900 $\pm$ 0.050 & 14.831 $\pm$ 2.134 & 0.898 $\pm$ 0.051 \\
$[0.4, 0.6]$ & 14.376 $\pm$ 0.571 & 0.894 $\pm$ 0.025 & 14.898 $\pm$ 2.172 & 0.900 $\pm$ 0.050 & 14.863 $\pm$ 2.153 & 0.900 $\pm$ 0.050 \\
$[0.5, 0.5]$ & 14.384 $\pm$ 0.573 & 0.895 $\pm$ 0.025 & 14.898 $\pm$ 2.172 & 0.900 $\pm$ 0.050 & 14.888 $\pm$ 2.170 & 0.900 $\pm$ 0.050 \\
$[0.6, 0.4]$ & 14.389 $\pm$ 0.574 & 0.895 $\pm$ 0.025 & 14.898 $\pm$ 2.172 & 0.900 $\pm$ 0.050 & 14.897 $\pm$ 2.172 & 0.900 $\pm$ 0.050 \\
$[0.7, 0.3]$ & 14.389 $\pm$ 0.574 & 0.895 $\pm$ 0.025 & 14.898 $\pm$ 2.172 & 0.900 $\pm$ 0.050 & 14.898 $\pm$ 2.172 & 0.900 $\pm$ 0.050 \\
$[0.8, 0.2]$ & 14.388 $\pm$ 0.575 & 0.895 $\pm$ 0.025 & 14.898 $\pm$ 2.172 & 0.900 $\pm$ 0.050 & 14.898 $\pm$ 2.172 & 0.900 $\pm$ 0.050 \\
$[0.9, 0.1]$ & 14.388 $\pm$ 0.574 & 0.895 $\pm$ 0.025 & 14.898 $\pm$ 2.172 & 0.900 $\pm$ 0.050 & 14.898 $\pm$ 2.172 & 0.900 $\pm$ 0.050 \\
\bottomrule
\end{tabular}}
\end{table}

\begin{table}[H]
\centering
\caption{Mean length and coverage for RF across 30 seeds for different weights, Boston dataset, $\alpha=0.1, K=5$.}
\label{tab:rf_weights}
\resizebox{\textwidth}{!}{
\begin{tabular}{lcccccc}
\toprule
$[w,1-w]$ & \multicolumn{2}{c}{ECCP} & \multicolumn{2}{c}{ECCP-exch} & \multicolumn{2}{c}{UR-ECCP-exch} \\
\cmidrule(lr){2-3} \cmidrule(lr){4-5} \cmidrule(lr){6-7}
 & Len. & Cov. & Len. & Cov. & Len. & Cov. \\
\midrule
$[0.05, 0.95]$ & 9.439 $\pm$ 0.659 & 0.884 $\pm$ 0.031 & 9.509 $\pm$ 1.530 & 0.886 $\pm$ 0.044 & 9.456 $\pm$ 1.502 & 0.884 $\pm$ 0.043 \\
$[0.1, 0.9]$ & 9.439 $\pm$ 0.659 & 0.884 $\pm$ 0.031 & 9.509 $\pm$ 1.530 & 0.886 $\pm$ 0.044 & 9.456 $\pm$ 1.502 & 0.884 $\pm$ 0.043 \\
$[0.2, 0.8]$ & 9.439 $\pm$ 0.659 & 0.884 $\pm$ 0.031 & 9.509 $\pm$ 1.530 & 0.886 $\pm$ 0.044 & 9.457 $\pm$ 1.502 & 0.884 $\pm$ 0.043 \\
$[0.3, 0.7]$ & 9.444 $\pm$ 0.663 & 0.884 $\pm$ 0.031 & 9.509 $\pm$ 1.530 & 0.886 $\pm$ 0.044 & 9.472 $\pm$ 1.513 & 0.885 $\pm$ 0.043 \\
$[0.4, 0.6]$ & 9.457 $\pm$ 0.669 & 0.884 $\pm$ 0.031 & 9.509 $\pm$ 1.530 & 0.886 $\pm$ 0.044 & 9.488 $\pm$ 1.520 & 0.885 $\pm$ 0.043 \\
$[0.5, 0.5]$ & 9.468 $\pm$ 0.674 & 0.885 $\pm$ 0.032 & 9.509 $\pm$ 1.530 & 0.886 $\pm$ 0.044 & 9.502 $\pm$ 1.529 & 0.886 $\pm$ 0.044 \\
$[0.6, 0.4]$ & 9.469 $\pm$ 0.676 & 0.884 $\pm$ 0.031 & 9.509 $\pm$ 1.530 & 0.886 $\pm$ 0.044 & 9.508 $\pm$ 1.530 & 0.886 $\pm$ 0.044 \\
$[0.7, 0.3]$ & 9.469 $\pm$ 0.677 & 0.884 $\pm$ 0.032 & 9.509 $\pm$ 1.530 & 0.886 $\pm$ 0.044 & 9.509 $\pm$ 1.530 & 0.886 $\pm$ 0.044 \\
$[0.8, 0.2]$ & 9.469 $\pm$ 0.677 & 0.884 $\pm$ 0.032 & 9.509 $\pm$ 1.530 & 0.886 $\pm$ 0.044 & 9.509 $\pm$ 1.530 & 0.886 $\pm$ 0.044 \\
$[0.9, 0.1]$ & 9.469 $\pm$ 0.676 & 0.884 $\pm$ 0.032 & 9.509 $\pm$ 1.530 & 0.886 $\pm$ 0.044 & 9.509 $\pm$ 1.530 & 0.886 $\pm$ 0.044 \\
\bottomrule
\end{tabular}}
\end{table}

\begin{table}[H]
\centering
\caption{Mean length and coverage for Lasso across 30 seeds for different weights, Boston dataset, $\alpha=0.1, K=5$.}
\label{tab:lasso_weights}
\resizebox{\textwidth}{!}{
\begin{tabular}{lcccccc}
\toprule
$[w,1-w]$ & \multicolumn{2}{c}{ECCP} & \multicolumn{2}{c}{ECCP-exch} & \multicolumn{2}{c}{UR-ECCP-exch} \\
\cmidrule(lr){2-3} \cmidrule(lr){4-5} \cmidrule(lr){6-7}
 & Len. & Cov. & Len. & Cov. & Len. & Cov. \\
\midrule
$[0.05, 0.95]$ & 14.393 $\pm$ 0.576 & 0.896 $\pm$ 0.026 & 14.887 $\pm$ 2.108 & 0.903 $\pm$ 0.047 & 14.803 $\pm$ 2.078 & 0.901 $\pm$ 0.051 \\
$[0.1, 0.9]$ & 14.393 $\pm$ 0.576 & 0.896 $\pm$ 0.026 & 14.887 $\pm$ 2.108 & 0.903 $\pm$ 0.047 & 14.803 $\pm$ 2.078 & 0.901 $\pm$ 0.051 \\
$[0.2, 0.8]$ & 14.394 $\pm$ 0.577 & 0.896 $\pm$ 0.026 & 14.887 $\pm$ 2.108 & 0.903 $\pm$ 0.047 & 14.803 $\pm$ 2.079 & 0.901 $\pm$ 0.051 \\
$[0.3, 0.7]$ & 14.405 $\pm$ 0.580 & 0.896 $\pm$ 0.025 & 14.887 $\pm$ 2.108 & 0.903 $\pm$ 0.047 & 14.820 $\pm$ 2.084 & 0.901 $\pm$ 0.051 \\
$[0.4, 0.6]$ & 14.423 $\pm$ 0.582 & 0.896 $\pm$ 0.025 & 14.887 $\pm$ 2.108 & 0.903 $\pm$ 0.047 & 14.854 $\pm$ 2.098 & 0.902 $\pm$ 0.048 \\
$[0.5, 0.5]$ & 14.433 $\pm$ 0.582 & 0.896 $\pm$ 0.026 & 14.887 $\pm$ 2.108 & 0.903 $\pm$ 0.047 & 14.877 $\pm$ 2.110 & 0.903 $\pm$ 0.049 \\
$[0.6, 0.4]$ & 14.437 $\pm$ 0.582 & 0.896 $\pm$ 0.026 & 14.887 $\pm$ 2.108 & 0.903 $\pm$ 0.047 & 14.886 $\pm$ 2.109 & 0.903 $\pm$ 0.047 \\
$[0.7, 0.3]$ & 14.437 $\pm$ 0.582 & 0.896 $\pm$ 0.026 & 14.887 $\pm$ 2.108 & 0.903 $\pm$ 0.047 & 14.887 $\pm$ 2.109 & 0.903 $\pm$ 0.047 \\
$[0.8, 0.2]$ & 14.436 $\pm$ 0.583 & 0.896 $\pm$ 0.026 & 14.887 $\pm$ 2.108 & 0.903 $\pm$ 0.047 & 14.887 $\pm$ 2.108 & 0.903 $\pm$ 0.047 \\
$[0.9, 0.1]$ & 14.436 $\pm$ 0.583 & 0.896 $\pm$ 0.026 & 14.887 $\pm$ 2.108 & 0.903 $\pm$ 0.047 & 14.887 $\pm$ 2.108 & 0.903 $\pm$ 0.047 \\
\bottomrule
\end{tabular}}
\end{table}

\section{CA Experiments \& Details}
\label{app:agg_random_splits}

\paragraph{Experimental setup.}
For each dataset and random seed, we split the data into $50\%$ training, $35\%$ calibration, and $15\%$ test sets. 
A collection of $K=7$ regression models $\{\widehat{\mu}_k\}_{k=1}^K$ is trained on the same training set: Linear Regression, regularized Lasso, iterative SGDRegressor, a probabilistic model (Bayesian Ridge) and tree-based ensembles (Random Forest and HistGradientBoosting), and a neural network (MLPRegressor). 
To study heterogeneous calibration sizes, each model $k$ is assigned its own calibration subset 
$S_{\mathrm{cal}}^{(k)} \subseteq S_{\mathrm{cal}}$, sampled uniformly from the common calibration pool, with a random size between $60\%$ and $100\%$ of $|S_{\mathrm{cal}}|$. 
For model $k$, conformity scores are computed as $R_j^{(k)} = |Y_j - \widehat{\mu}_k(X_j)|, 
(X_j,Y_j) \in S_{\mathrm{cal}}^{(k)}.$
Given a test point $x$ and candidate response $y$, the model-specific conformal p-value is
\[
p_k(y) =
\frac{1 + \sum_{j \in S_{\mathrm{cal}}^{(k)}} 
\mathbf{1}\{R_j^{(k)} \ge |y-\widehat{\mu}_k(x)|\}}
{|S_{\mathrm{cal}}^{(k)}|+1},
\]

which is calibrated into an e-variable using the calibrators we consider in section \ref{sec:experiments}.
All methods are evaluated at nominal miscoverage level $\alpha=0.05$ over multiple random seeds. 
Coverage is computed on the test responses, while average prediction-set length is approximated on a finite grid over a data-adaptive interval around the model predictions.

\subsection{Theoretical Validity of WECA}

\label{app:weca_validity}

Although, by design, all of our e-value-based methods satisfy the coverage guarantee, we explain here how we implement our WECA method in order for the coverage guarantee to be preserved. The core idea is to choose weights that are independent of a calibration data split, so that the aggregated e-variable
\[
E^\omega := \sum_{k=1}^K \omega_k E_k
\]
yields prediction sets with small average length.

Consider $K$ predictors indexed by $k\in[K]$. Each predictor $k$ may have its own calibration index set, denoted by $I_{\mathrm{cal}}^{(k)}$. For a set of calibration indices $J_k\subseteq I_{\mathrm{cal}}^{(k)}$, let
$P_k^{J_k}(x,y)$
denote the conformal p-variable constructed from predictor $k$ using only the calibration data indexed by $J_k$, evaluated at test input $x$ and candidate label $y\in\mathcal Y$. Its calibrated e-variable is
\[
E_k^{J_k}(x,y)
:=
F_{|J_k|,\alpha}\big(P_k^{J_k}(x,y)\big),
\]
where $F_{|J_k|,\alpha}$ is the P2E calibrator corresponding to calibration size $|J_k|$ and miscoverage level $\alpha$.

Denote the space of weight vectors $\omega=(\omega_1,\dots,\omega_K)$ by the simplex
\[
\Delta_K
:=
\left\{
\omega\in\mathbb R_+^K:
\sum_{k=1}^K \omega_k=1
\right\}.
\]
For any $\omega\in\Delta_K$ and any vector of calibration index sets
$\mathbf J=(J_1,\ldots,J_K)$, define the weighted e-variable
\[
E_\omega^{\mathbf J}(x,y)
:=
\sum_{k=1}^K \omega_k E_k^{J_k}(x,y),
\]
and the corresponding prediction set
\[
\mathcal C_{\alpha}^{\omega,\mathbf J}(x)
:=
\left\{
y\in\mathcal Y:
E_\omega^{\mathbf J}(x,y)<\frac1\alpha
\right\}.
\]

We split the calibration data into a \emph{tuning} part, used only to select the weights, and an \emph{inference} part, used only to produce the final conformal set:
\[
I_{\mathrm{cal}}
=
I_{\mathrm{tune}}\sqcup I_{\mathrm{inf}},
\qquad
|I_{\mathrm{tune}}|=n_{\mathrm{tune}},
\qquad
|I_{\mathrm{inf}}|=n_{\mathrm{inf}}.
\]
We then split the tuning indices again:
\[
I_{\mathrm{tune}}
=
I_{\mathrm{tune},1}
\sqcup
I_{\mathrm{tune},2},
\qquad
|I_{\mathrm{tune},1}|=n_{\mathrm{tune},1},
\qquad
|I_{\mathrm{tune},2}|=n_{\mathrm{tune},2}.
\]

Since each model may have its own calibration set, we define the model-specific splits
\[
J_{k,\mathrm{tune},1}
:=
I_{\mathrm{cal}}^{(k)}\cap I_{\mathrm{tune},1},
\qquad
J_{k,\mathrm{inf}}
:=
I_{\mathrm{cal}}^{(k)}\cap I_{\mathrm{inf}},
\]
with sizes
\[
n_{k,\mathrm{tune},1}
:=
|J_{k,\mathrm{tune},1}|,
\qquad
n_{k,\mathrm{inf}}
:=
|J_{k,\mathrm{inf}}|.
\]
We also write
\[
\mathbf J_{\mathrm{tune},1}
:=
(J_{1,\mathrm{tune},1},\ldots,J_{K,\mathrm{tune},1}),
\qquad
\mathbf J_{\mathrm{inf}}
:=
(J_{1,\mathrm{inf}},\ldots,J_{K,\mathrm{inf}}).
\]

The first tuning split $I_{\mathrm{tune},1}$ is used as a calibration set to construct the e-variables, while the second tuning split $I_{\mathrm{tune},2}$ is used only to evaluate the empirical prediction-set size. We choose
\[
\omega^\star
\in
\arg\min_{\omega\in\Delta_K}
\frac1{n_{\mathrm{tune},2}}
\sum_{i\in I_{\mathrm{tune},2}}
\left|
\mathcal C_{\alpha}^{\omega,\mathbf J_{\mathrm{tune},1}}(X_i)
\right|.
\]

After the weights have been selected, the final WECA set is constructed using only the inference calibration split:
\[
\mathcal C_{\alpha}^{\mathrm{WECA}}(X_{n+1})
:=
\left\{
y\in\mathcal Y:
E_{\mathrm{inf}}(X_{n+1},y)<\frac1\alpha
\right\},
\]
where
\[
E_{\mathrm{inf}}(X_{n+1},y)
:=
E_{\omega^\star}^{\mathbf J_{\mathrm{inf}}}(X_{n+1},y)
=
\sum_{k=1}^K
\omega_k^\star
E_k^{J_{k,\mathrm{inf}}}(X_{n+1},y),
\]
with
\[
E_k^{J_{k,\mathrm{inf}}}(X_{n+1},y)
=
F_{n_{k,\mathrm{inf}},\alpha}
\big(
P_k^{J_{k,\mathrm{inf}}}(X_{n+1},y)
\big).
\]

We now show that the final statistic is an e-variable, thanks to the independence of $\omega^*$ from the inference split and the test point:
\[
\mathbb{E}_{\mathbb{H}_0}[E_{\text{inf}}]
=
\sum_{k=1}^K
\mathbb{E}_{\mathbb{H}_0}
\left[
\omega_k^*
E_k^{J_{k,\mathrm{inf}}}(X_{n+1},Y_{n+1})
\right]
=
\sum_{k=1}^K
\mathbb{E}_{\mathbb{H}_0}[\omega_k^*]
\underbrace{
\mathbb{E}_{\mathbb{H}_0}
\left[
E_k^{J_{k,\mathrm{inf}}}(X_{n+1},Y_{n+1})
\right]
}_{=1}
=
\sum_{k=1}^K
\mathbb{E}_{\mathbb{H}_0}[\omega_k^*]
=
1 .
\]
This proves that WECA preserves the finite-sample coverage guarantee. The randomized version UR-WECA is obtained by replacing the threshold $1/\alpha$ with $U/\alpha$, where $U\sim\mathrm{Unif}(0,1)$ is independent of all data. Its validity follows analogously from the UR-MI.

\section{Discussion and Future Directions}
\label{app:future_directions}

Our methodology enables a simple and general conversion of conformal p-variables into e-variables using any P2E calibrator. 
Once e-values are obtained, our framework naturally lends itself to several extensions and applications. 
Beyond the settings discussed above, we highlight a few potential directions for future research.


\paragraph{Post-hoc selection.}
E-values naturally support post-hoc (data-adaptive) inference. 
Following \citet{koning2023post}, a nonnegative random variable \(P\) is \emph{post-hoc valid} 
if, for every (possibly data-dependent) random threshold \(\delta > 0\),
\[
\sup_{\delta > 0} 
\mathbb{E}\!\left[\frac{\mathbf{1}\{P \le \delta\}}{\delta}\right] \le 1.
\]
Koning further shows that this property holds if and only if \(1/P\) is an e-variable. 
In our notation, for fixed \(\alpha \in (0,1)\) and taking $P:=1/E_{n,\alpha}$,
the post-hoc validity condition can equivalently be written as
\[
\sup_{\delta > 0} 
\mathbb{E}\!\left[\frac{\mathbf{1}\!\left\{E_{n,\alpha} \ge 1/\delta\right\}}{\delta}\right] 
\le 1,
\]
allowing data-dependent, post-hoc choices of \(\delta\).

\paragraph{Randomized e-CP.}
As previously discussed, both p-CP and e-CP prediction sets coincide under P2E calibration:
\[
\{P_n > \alpha\} = \{E_{n,\alpha} < 1/\alpha\},
\]
thus preserving the conformal coverage guarantee. 
One can further consider a randomized variant using an independent \(U \sim \mathrm{Unif}(0,1)\), 
where a candidate is included if \(E_n < U/\alpha\). 
This rule defines a prediction set that is nested within the nonrandomized version 
and remains marginally valid by the UR-MI property:
\[
\{E_{n,\alpha} < U/\alpha\} \subset 
\{E_{n,\alpha} < 1/\alpha\} = \{P_n > \alpha\},
\]
with all satisfying the  \(1 - \alpha\) coverage guarantee. 
Exploring how such randomization affects the distribution of coverage and the efficiency 
of the resulting prediction sets represents an interesting direction for future work.

\end{document}